\documentclass[10pt,twocolumn,letterpaper]{article}

\usepackage{cvpr}              

\usepackage{amsmath, amsthm, amssymb}
\usepackage{lineno}
\usepackage{makecell}
\newtheorem{theorem}{Theorem}
\newtheorem{lemma}{Lemma}
\newtheorem{assumption}{Assumption}

\newtheorem{remark}{Remark}
\usepackage{mathtools}
\usepackage[ruled]{algorithm2e} 
\usepackage{multirow} 

\newcommand{\norm}[1]{\left\| #1 \right\|}
\newcommand{\inner}[2]{\left\langle #1, #2 \right\rangle}

\definecolor{cvprblue}{rgb}{0.21,0.49,0.74}
\usepackage[pagebackref,breaklinks,colorlinks,allcolors=cvprblue]{hyperref}

\def\paperID{33300} 
\def\confName{CVPR}
\def\confYear{2026}

\title{Submodel Extraction for Efficient and Personalized Federated Learning via Optimal Transport}

\author{
    \textbf{Zheng Jiang}$^{1}$,
    \textbf{Nan He}$^{2}$,
    \textbf{Yiming Chen}$^{1}$,
    \textbf{Lifeng Sun}$^{1,3}$\thanks{Corresponding Author.}
    \\
    $^1$Tsinghua University \,
    $^2$Beijing University of Technology \\
    $^3$Key Laboratory of Pervasive Computing, Ministry of Education \\
    \texttt{jz24@mails.tsinghua.edu.cn, sunlf@tsinghua.edu.cn}
}

\begin{document}
\maketitle
\begin{abstract}
Federated Learning (FL) enables collaborative model training while preserving data privacy, but its practical deployment is hampered by system and statistical heterogeneity. While federated network pruning offers a path to mitigate these issues, existing methods face a critical dilemma: server-side pruning lacks personalization, whereas client-side pruning is computationally prohibitive for resource-constrained devices. Furthermore, the pruning process itself induces significant parametric divergence among heterogeneous submodels, destabilizing training and hindering global convergence. 
To address these challenges, we propose \textbf{SubFLOT}, a novel framework for server-side personalized federated pruning. SubFLOT introduces an Optimal Transport-enhanced Pruning (OTP) module that treats historical client models as proxies for local data distributions, formulating the pruning task as a Wasserstein distance minimization problem to generate customized submodels without accessing raw data. Concurrently, to counteract parametric divergence, our Scaling-based Adaptive Regularization (SAR) module adaptively penalizes a submodel's deviation from the global model, with the penalty's strength scaled by the client's pruning rate. Comprehensive experiments demonstrate that SubFLOT consistently and substantially outperforms state-of-the-art methods, underscoring its potential for deploying efficient and personalized models on resource-constrained edge devices.
\end{abstract}    
\section{Introduction}
\label{sec:intro}
Federated Learning (FL)~\cite{mcmahan2017communication} is a promising distributed machine learning paradigm that enables collaborative training across multiple devices while preserving data confidentiality. 
However, traditional FL methods encounter significant challenges in practical deployment: (1) System Heterogeneity. The available system resources usually vary among different devices, which restricts the feasible size of the shared models, resulting in training delays and resource underutilization. (2) Statistical Heterogeneity. The non-IID nature of local data distributions often manifests as feature shifts and label imbalance, inevitably leading to performance degradation and training instability~\cite{yang2023efficient,he2025dynfed,chen2024fedawa}.

To navigate these impediments, federated network pruning has been recognized as a potent strategy, offering the dual benefits of computational efficiency and model personalization~\cite{Diao2020HeteroFLCA,jiang2022fedmp,kim2022depthfl}. By allowing individual clients to train sparse, architecture-agnostic submodels tailored to their specific hardware capabilities, this approach markedly reduces both computational and communication overheads. Furthermore, by leveraging data-dependent pruning criteria, it facilitates the creation of personalized models that are better adapted to local data distributions, thereby mitigating the adverse effects of statistical heterogeneity.

However, the efficacy of existing federated pruning methodologies is constrained by two critical, unresolved issues. The first is a central dilemma concerning the \textit{locus of the pruning decision}, which creates a stark trade-off between server-enforced uniformity and client-driven personalization~\cite{yi2024fedp3}. On one hand, server-side global pruning strategies apply a uniform compression policy but remain oblivious to local data specificities due to privacy constraints, thus failing to achieve meaningful personalization. On the other hand, client-side local pruning, often following a ``train-prune-finetune" paradigm~\cite{frankle2018lottery}, can yield highly customized models. Yet, this approach imposes prohibitive computational burdens on resource-constrained clients, which must initially handle the full-sized model. This deadlock exposes a fundamental research question: \textit{How to achieve personalized pruning on the server without direct access to local data, while simultaneously addressing system and statistical heterogeneity?}

Secondly, the act of pruning can paradoxically exacerbate heterogeneity by inducing significant \textit{parametric divergence} among client submodels. As observed in prior works~\cite{Diao2020HeteroFLCA,deng2022tailorfl}, the weight distributions of submodels diverge from the global model in a manner correlated with their sparsity. Specifically, submodels subjected to higher pruning rates tend to develop weight magnitudes of a larger scale, causing them to deviate substantially from the parametric distribution of their less-pruned counterparts and the global model. This pruning-induced parametric drift not only destabilizes local training dynamics but also hampers the convergence of the global model during aggregation. This leads to a second pivotal research question: \textit{How to formulate a local objective that adaptively regularizes submodel divergence to harmonize training dynamics?}

To overcome these limitations, we propose a novel method \textit{\textbf{Sub}model Extraction for Efficient and Personalized \textbf{F}ederated \textbf{L}earning via \textbf{O}ptimal \textbf{T}ransport} (\textbf{SubFLOT}) to holistically address system and statistical heterogeneity. 
As illustrated in Figure~\ref{fig:pipeline}, SubFLOT introduces an Optimal Transport-enhanced Pruning (OTP) module to achieve server-side personalization. Recognizing that direct access to local data is prohibited, OTP astutely leverages historical client models as high-fidelity proxies for local data distributions. By minimizing the Wasserstein distance between the global model and these historical models, OTP computes a transport plan that guides a pruning process tailored to each client's data. Remarkably, this mechanism exhibits a powerful duality: during aggregation, it is repurposed as the OT-enhanced Aggregation (OTA) module to align the parameter spaces of disparate submodels.
Concurrently, to counteract the pruning-induced parameter divergence, SubFLOT incorporates a Scaling-based Adaptive Regularization (SAR) module. SAR imposes an adaptive penalty on a submodel's deviation from the global model, with the penalty's intensity dynamically scaled by the client's pruning rate. This unique design stabilizes the local training of heavily pruned models and accelerates global convergence.

Our principal contributions are summarized as follows:
\begin{itemize}
    \item We propose SubFLOT, to the best of our knowledge, the first framework to systematically address server-side personalized pruning while concurrently managing the dual challenges of feature-space and parameter-space heterogeneity in federated submodel training.
    \item We introduce a novel Optimal Transport-driven methodology that unifies personalized pruning and heterogeneous aggregation into a coherent parameter-space alignment problem, significantly enhancing model performance and personalization.
    \item We design a Scaling-based Adaptive Regularization (SAR) module that dynamically constrains submodel divergence based on the pruning ratio, fostering more stable local training and accelerating global convergence.
    \item We conduct extensive experiments across diverse datasets and non-IID settings, demonstrating that SubFLOT consistently and substantially outperforms state-of-the-art methods, highlighting its efficacy for practical deployment on resource-constrained devices.
\end{itemize}
\section{Related Work}
\label{sec:related_work}
\paragraph{Federated Network Pruning}
Federated pruning research grapples with a trade-off between server-side and client-side strategies~\cite{yi2024fedp3}. Server-side methods~\cite{jiang2022fedmp,wang2024feddse} offer communication efficiency by distributing pre-pruned, uniform submodels, but their data-agnostic nature limits their ability to achieve personalization in non-IID settings. Conversely, client-side approaches~\cite{li2021fedmask} achieve personalization by pruning locally based on client data, but at the cost of high computational overhead and potential privacy risks. SubFLOT resolves this dilemma by introducing a new paradigm: server-side personalized pruning. By leveraging Optimal Transport to match parameter distributions, it achieves data-aware personalization on the server, eliminating the burdens of client-side computation.

\paragraph{Optimal Transport}
The geometric foundation of Optimal Transport (OT)~\cite{kantorovich2006translocation} has catalyzed its adoption across machine learning for tasks requiring robust distribution alignment ~\cite{damodaran2018deepjdot,fatras2021unbalanced}. 
Within federated learning, recent efforts like FedOTP~\cite{li2024global} and FedAli~\cite{ek2024fedali} have successfully applied OT to enhance model personalization. However, these methods typically perform feature-space alignment on the client side, which incurs substantial computational and privacy overhead. Our work diverges fundamentally in two key aspects. First, we pioneer the application of OT to the task of federated network pruning, a novel use case. Second, we shift the locus of alignment from the client-side feature space to the server-side parameter space. This key innovation enables SubFLOT to perform sophisticated personalization entirely on the server, establishing a more efficient and privacy-preserving framework.
\section{Methodology}
\label{sec:methodology}

\begin{figure*}[t]
\centering
\includegraphics[width=1\linewidth]{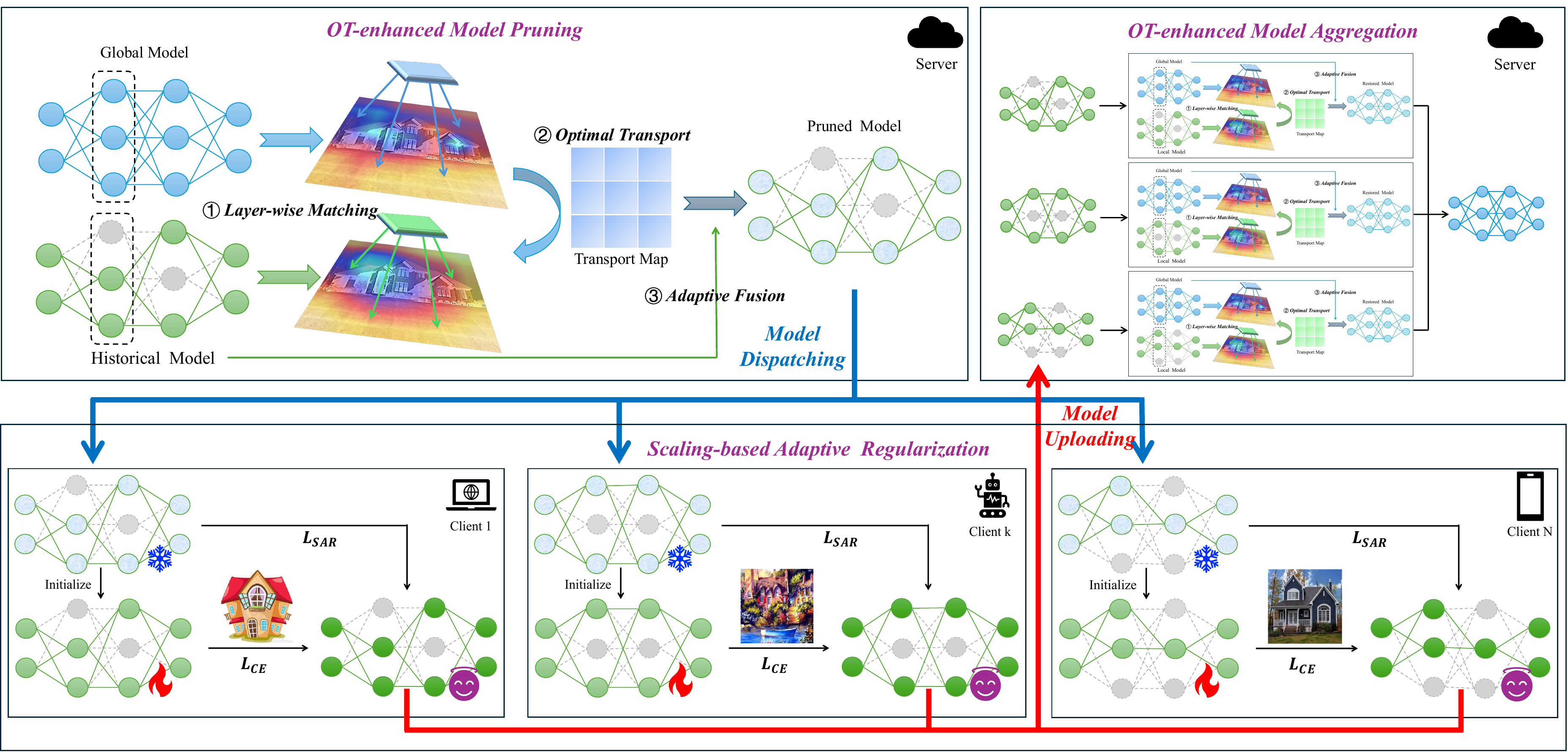} 
\caption{An overview of the proposed SubFLOT framework. On the server, Optimal Transport-enhanced Pruning (OTP) leverages historical models to generate personalized, heterogeneous submodels for clients. Clients then train these submodels locally, guided by Scaling-based Adaptive Regularization (SAR). Finally, the server utilizes the repurposed OT-enhanced Aggregation (OTA) module to align and aggregate the updated submodels, improving global model convergence.} 
\label{fig:pipeline}
\vspace{-2mm}
\end{figure*}

\subsection{Preliminaries}
\label{preliminaries}
We consider a federated network of $N$ clients and each client $i$($i \in [N]$) holds a local dataset $\mathcal{D}_i$. For the personalized task,the local objective for client $i$ is defined as:
\begin{equation}
\label{eq:local_func}
\mathcal{F}_i(W_i) \coloneqq \mathbb{E}_{\xi_i \sim \mathcal{D}_i} [\mathcal{L}_i(W_i; \xi_i)],
\end{equation}
where $W_i$ represents the parameters of the personalized model for client $i$, and $\mathcal{L}_i$ is the client-specific loss function evaluated on a data sample $\xi_i$. The global objective of the pFL system is to find the optimal set of personalized models $\{W_1^*, \dots, W_N^*\}$ that minimizes the weighted average of these local objectives:
\begin{equation}
\label{eq:global_func}
\min_{W_1, \dots, W_N} \mathcal{F}(W_1, \dots, W_N) \coloneqq \sum_{i=1}^N p_i \mathcal{F}_i(W_i),
\end{equation}
where the weights $p_i = |\mathcal{D}_i| / \sum_{j=1}^N |\mathcal{D}_j|$ are proportional to the size of the local datasets and satisfy $\sum_{i=1}^N p_i = 1$. For convergence analysis, we also define the global loss function $\mathcal{F}(W_G) \coloneqq \sum_{i=1}^N p_i \mathcal{F}_i(W_G)$, which measures the performance of a single global model $W_G$ across all clients. The optimal global model is denoted as $W^* = \arg\min_W \mathcal{F}(W_G)$.

\subsection{OT-Enhanced Model Pruning(OTP)}
\label{subsec:OTP}

A central challenge in pFL is the server's inability to generate client-specific models without access to local data. We address this by leveraging the principle that a client's historical model parameters serve as a proxy, implicitly encoding its local data distribution. Consequently, we formulate the server-side personalized pruning task as an optimal transport problem between the parameters of the global model $W_G$ and each client's historical model $W_i$, with the objective of aligning functionally equivalent neurons.

Directly computing the Optimal Transport plan for entire deep neural networks is computationally intractable due to their high dimensionality. To circumvent this, we introduce a \textit{progressive layer-wise matching} strategy. This approach decomposes the global OT problem into a sequence of manageable, layer-specific transport problems, iteratively computing transport plans $\{T_i^{(l)}\}_{l=1}^L$ for each layer $l$. 

For a given client $i$ and layer $l$, let $W_G^{(l, l-1)} \in \mathbb{R}^{d_l \times d_{l-1}}$ and $W_i^{(l, l-1)} \in \mathbb{R}^{d'_l \times d'_{l-1}}$ represent the weight matrices connecting layer $l-1$ to layer $l$. The process unfolds as follows:
First, given the transport plan $T_i^{(l-1)} \in \mathbb{R}^{d_{l-1} \times d'_{l-1}}$ from the previous layer, which aligns the neurons of layer $l-1$, we first remap the input space of the global model's weights to match the client's aligned feature space:
\begin{equation}
\label{eq:align_input}
\widehat{W}_G^{(l,l-1)} = W_G^{(l,l-1)} T_i^{(l-1)}.
\end{equation}
Next, we treat the output neurons of the aligned global weights $\widehat{W}_G^{(l,l-1)}$ and the client's local weights $W_i^{(l,l-1)}$ as two discrete probability distributions, $\mu^{(l)}$ and $\nu^{(l)}$, respectively (assuming uniform weights on neurons for simplicity). We compute a cost matrix $C^{(l)} \in \mathbb{R}^{d_l \times d'_l}$, where each entry $C_{jk}^{(l)}$ is the pairwise Euclidean distance between the $j$-th neuron of $\widehat{W}_G^{(l,l-1)}$ and the $k$-th neuron of $W_i^{(l,l-1)}$. The optimal transport plan $T_i^{(l)}$ for the current layer is then found by solving the discrete OT problem:
\begin{equation}
\label{eq:otp_problem}
T_i^{(l)} = \arg\min_{T \in \Pi(\mu^{(l)}, \nu^{(l)})} \langle C^{(l)}, T \rangle_F,
\end{equation}
where $\Pi(\mu^{(l)}, \nu^{(l)})$ denotes the set of all valid transport plans (joint probability distributions) with marginals $\mu^{(l)}$ and $\nu^{(l)}$, and $\langle \cdot, \cdot \rangle_F$ is the Frobenius dot product.

Upon obtaining the optimal transport plan $T_i^{(l)}$, we first generate an aligned version of the global layer's weights by projecting them into the client's parameter space:
\begin{equation}
\label{eq:otp_align}
W_{\text{aligned}}^{(l,l-1)} = {T_i^{(l)}}^\top \widehat{W}_G^{(l,l-1)}.
\end{equation}
To produce the final personalized submodel $\widetilde{W}_i^{(l,l-1)}$ for the client, we fuse this aligned global knowledge with the client's specialized parameters from its historical model $W_i^{(l, l-1)}$:
\begin{equation}
\label{eq:otp_mix}
\widetilde{W}_i^{(l,l-1)} = \alpha \cdot W_{\text{aligned}}^{(l,l-1)} + (1 - \alpha) \cdot W_i^{(l,l-1)}.
\end{equation}
where $\alpha \in [0,1]$ is a hyperparameter that balances the trade-off between global knowledge transfer and local specialization. This process yields a personalized model $\widetilde{W}_i$ that is pre-adapted to client $i$'s data characteristics before local training commences.

\subsection{Scaling-based Adaptive Regularization(SAR)}
\label{subsec:SAR}
Having received its personalized submodel $\widetilde{W}_i$, each client performs local training. A key challenge here is that aggressive pruning can lead to significant parametric divergence, where smaller models undergo large weight updates, destabilizing training. To address this, we introduce Scaling-based Adaptive Regularization (SAR), a novel local objective function that constrains the training trajectory.

Unlike prior work such as HeteroFL~\cite{Diao2020HeteroFLCA}, which applies post-hoc corrections to weight magnitudes after local training, SAR proactively regularizes the training dynamics. This prevents the model from drifting into parameter regions that are detrimental to aggregation. Our approach is motivated by the insight that submodels with higher pruning rates are more susceptible to this divergence.

The core of SAR is a regularization term that penalizes the deviation of the client's current model $W_i$ from its server-provided anchor model $\widetilde{W}_i$, with the penalty strength adaptively scaled by the client's pruning rate $\rho_i$. The SAR loss for client $i$ is defined as:
\begin{equation}
\label{eq:sar_loss}
\mathcal{L}_{\text{SAR}}(W_i) = \rho_i \cdot \|W_i - \widetilde{W}_i\|_2^2.
\end{equation}
This formulation yields two crucial benefits: \textit{(1) Adaptive Control:} The penalty is directly proportional to the pruning rate $\rho_i$. Consequently, clients with smaller models (higher $\rho_i$), which are most at risk of divergence, receive stronger regularization, tethering their updates to the well-initialized anchor $\widetilde{W}_i$. \textit{(2) Divergence Mitigation:} By anchoring local training to the server-provided submodel, SAR discourages large parametric shifts. This mitigates the risk of feature distribution skew and ensures that the updated models remain in a cooperative region of the parameter space.
The complete local objective for client $i$ on its local dataset $\mathcal{D}_i$ thus integrates the standard cross-entropy loss $\mathcal{L}_{\text{CE}}$ with our SAR term:
\begin{equation}
\label{eq:final_loss}
\mathcal{L}_{i}(W_i) = \mathcal{L}_{\text{CE}}(W_i; \mathcal{D}_i) + \lambda \cdot \mathcal{L}_{\text{SAR}}(W_i),
\end{equation}
where $\lambda$ is a hyperparameter balancing task-specific learning and training stability. This strategy promotes more stable local training and produces higher-quality model updates for aggregation.

\subsection{OT-Enhanced Model Aggregation (OTA)}
\label{subsec:OTA}

The final component of our framework is the OT-enhanced Aggregation (OTA) module, which addresses feature space misalignment during model fusion. Standard federated averaging often falters in heterogeneous settings because it directly averages parameters that may correspond to semantically different features, leading to destructive interference.

The OTA module elegantly repurposes the geometric insights from OTP. For each client $i$, OTA computes a transport map $\mathcal{T}_i$ that aligns the parameters of the client's updated model $W_i^t$ back to the canonical parameter space of the global model $W_G^t$. This map is derived using the same progressive layer-wise OT procedure as in OTP, but in the reverse direction—matching the updated client model $W_i^t$ to the previous global model $W_G^t$.
The global model for the next round, $W_G^{t+1}$, is then updated by aggregating these aligned client models:
\begin{equation}
\label{eq:ota_agg}
W_G^{t + 1} = \sum_{i=1}^N p_i \cdot \mathcal{T}_i(W_i^t).
\end{equation}
By aligning models before averaging, OTA offers two significant advantages. First, it mitigates magnitude discrepancies arising from varied pruning rates, as the OT-based alignment inherently normalizes parameter scales. Second, it leverages the geometric properties of OT to match functionally analogous neurons, thereby suppressing the negative effects of feature shifts and leading to more stable and effective global model updates.

\section{Convergence Analysis}
\label{convergence}
We provide a rigorous convergence analysis for our proposed SubFLOT. The objective is to demonstrate that the sequence of global models, $\{W_G^t\}_{t=0}^T$, generated by the algorithm converges linearly to a neighborhood of the optimal solution $W^*$ that minimizes the global loss function.

\subsection{Assumptions}
Our analysis relies on the following standard assumptions in the federated learning literature~\cite{Hao_2025_CVPR,Zhang2023fedcp,zhang2023eliminating}. Regarding Assumptions \ref{ass:ota_error} and \ref{ass:otp_error}, to ensure consistent dimensions for the norm calculation, parameters absent in the submodel are considered masked to zero.

\begin{assumption}[L-smoothness]
\label{ass:smooth}
The local objective functions $\mathcal{F}_1, \dots, \mathcal{F}_N$ are all $L$-smooth, \ie, for all $W_1$ and $W_2$, we have \(\norm{\nabla \mathcal{F}_i(W_1) - \nabla \mathcal{F}_i(W_2)} \le L \norm{W_1 - W_2}\).
\end{assumption}

\begin{assumption}[$\mu$-strong convexity]
\label{ass:convex}
$\mathcal{F}_1, \dots, \mathcal{F}_N$ are all $\mu$-strongly convex, \ie, for all $W_1$ and $W_2$, we have \(\mathcal{F}_i(W_2) \ge \mathcal{F}_i(W_1) + \inner{\nabla \mathcal{F}_i(W_1)}{W_2 - W_1} + \frac{\mu}{2} \norm{W_2 - W_1}^2\).
\end{assumption}

\begin{assumption}[Bounded Gradient Variance]
\label{ass:variance}
The variance of the stochastic gradient on each client is bounded, \ie, for any client $i$ and model $W$, we have \(\mathbb{E}_{\xi_i \sim \mathcal{D}_i}[\norm{\nabla \mathcal{L}_i(W; \xi_i) - \nabla \mathcal{F}_i(W)}^2] \le \sigma^2.\)
\end{assumption}

\begin{assumption}[Bounded Stochastic Gradient]
\label{ass:heterogeneity}
The expected squared norm of the full local gradients is bounded, \ie, for all $i$ and $W$, we have \(\mathbb{E}[\norm{\nabla \mathcal{F}_i(W)}^2] \le G^2.\)
\end{assumption}

\begin{assumption}[Bounded OT Perturbation]
\label{ass:ota_error}
OTA aligns a local model $W_i^t$ to the global parameter space, which introduces a bounded perturbation: \( \mathbb{E}_i [\norm{\mathcal{T}_i(W_i^t) - W_i^t}^2] = \sum_{i=1}^N p_i \norm{\mathcal{T}_i(W_i^t) - W_i^t}^2 \le \delta_{OT}^2. \)
\end{assumption}

\begin{assumption}[Bounded Personalization Error]
\label{ass:otp_error}
OTP generates a personalized anchor $\widetilde{W}_i^t$ from the global model $W_G^t$. The expected squared deviation of this personalized anchor from the global model is bounded:
\( \mathbb{E}_i [\norm{\widetilde{W}_i^t - W_G^t}^2] = \sum_{i=1}^N p_i \norm{\widetilde{W}_i^t - W_G^t}^2 \le \delta_{P}^2.\)
\end{assumption}

\subsection{Convergence Theorem}

\begin{theorem}
\label{thm:main}
Under Assumptions \ref{ass:smooth}--\ref{ass:otp_error}, let the number of local epochs be $E$, $\rho_{\max} = \max_i \rho_i$ and the local learning rate be $\eta_l \le \min \{\frac{1}{8LE},\;\frac{1}{4\lambda\rho_{\max}}\}$. After $T$ communication rounds, SubFLOT achieves the following guarantee:
\begin{equation}
    \mathbb{E}[\mathcal{F}(W_G^T) - \mathcal{F}(W_G^*)] \le \gamma^T (\mathcal{F}(W_G^0) - \mathcal{F}(W_G^*)) + \frac{\mathcal{E}}{\mu},
\end{equation}
where \(\gamma = 1 - \frac{\mu \eta_l E}{2} \) and the asymptotic error term $\mathcal{E}$ is given by:
\begin{equation}
\begin{split}
  \mathcal{E}
  =&\; \frac{L\eta_l E\,\sigma^2}{2}
     + 4L^2 E^2 \eta_l G^2
     + \frac{16 L E \eta_l }{\lambda\rho_{\max}}(\sigma^2 + G^2) \\
   &\; + \frac{5}{2}\,\delta_{OT}^2
     + 4 L \delta_P^2
     + 4 L E^2 \eta_l^2 (\lambda\rho_{\max})^2 \delta_P^2.
\end{split}
\end{equation}
\end{theorem}

\begin{remark}
\Cref{thm:main} establishes that \texttt{SubFLOT} converges at a rate of $1 - \mu \eta_l E / 2$ to a neighborhood of the global optimum $W_G^*$. The neighborhood size is determined by the stochastic gradient noise ($\sigma^2$), the statistical heterogeneity ($G^2$), the personalization deviation ($\delta_P^2$), the OT alignment perturbation ($\delta_{OT}^2$), and the strength of the SAR regularization ($\lambda\rho_{\max}$). The detailed proof of \Cref{thm:main} is provided in the Appendix \cref{suppl:conver}.
\end{remark}
\section{Experiments}
\label{sec:experiments}

\subsection{Experimental Setup}
\paragraph{Datasets and Partition}
Our evaluation spans multiple domains to ensure broad applicability.
For CV tasks, we use five public datasets: CIFAR10 and CIFAR100~\cite{krizhevsky2009learning}, TinyImageNet~\cite{chrabaszcz2017downsampled}, and two multi-domain datasets, Digit5~\cite{lecun1998gradient} and PACS~\cite{yu2022pacs}. For NLP tasks, we use the AG News~\cite{zhang2015character}. For IoT tasks, we utilize a Human Activity Recognition (HAR) dataset~\cite{anguita2012human} based on sensor signals. To simulate diverse statistical heterogeneity, we construct scenarios for label skew, feature shift, and real-world partitions. The details are provided in the Appendix \cref{suppl:exp}.

\paragraph{Backbones}
In line with established benchmarks and to demonstrate model-agnosticism, we pair each dataset with a special neural network architecture as follows: VGG11~\cite{simonyan2014very} for CIFAR10, MobileNetV2~\cite{sandler2018mobilenetv2} for CIFAR-100, Digit5CNN~\cite{pmlr-v139-feng21f} for Digit5, ResNet18~\cite{he2016deep} for both TinyImageNet and PACS, fastText~\cite{joulinetal2017bag} for AG News, and a specialized HARCNN~\cite{zeng2014convolutional} for HAR. 
\begin{table*}[ht]
\centering
\vspace{-1mm}
\caption{Experimental results(\%) under label skew settings and real-world settings.}
\vspace{-1mm}
\label{table:result_label_shift}
\resizebox{\linewidth}{!}{
\begin{tabular}{l|ccc|ccc|c|c}
\toprule
Settings & \multicolumn{3}{c|}{Pathological Label Skew} & \multicolumn{4}{c|}{Practical Label Skew(Dir)} & Real-world\\
\midrule
Tasks & \multicolumn{3}{c|}{CV} & \multicolumn{3}{c|}{CV} & \multicolumn{1}{c|}{NLP} & IoT \\
\midrule
Datasets   & CIFAR10 & CIFAR100 & TinyImageNet & CIFAR10 & CIFAR100 & TinyImageNet & AG News & HAR \\ 
\midrule
HeteroFL & 84.54 $\pm$ 0.10 & 40.95 $\pm$ 0.31 & 19.68 $\pm$ 0.10 & 81.21 $\pm$ 0.20 & 29.36 $\pm$ 0.32 & 19.91 $\pm$ 0.14 & 84.12 $\pm$ 0.11 & 69.80 $\pm$ 0.27 \\
FedRolex & 84.33 $\pm$ 0.27 & 43.36 $\pm$ 0.43 & 19.56 $\pm$ 0.12 & 81.04 $\pm$ 0.21 & 30.59 $\pm$ 0.39 & 20.91 $\pm$ 0.21 & 84.29 $\pm$ 0.34 & 68.96 $\pm$ 0.13 \\
FedDrop & 80.63 $\pm$ 0.22 & 28.20 $\pm$ 0.51 & 11.06 $\pm$ 1.22  & 78.49 $\pm$ 0.32 & 21.40 $\pm$ 0.73 & 12.93 $\pm$ 0.42 & 73.34 $\pm$ 0.58 & 57.48 $\pm$ 0.54 \\
FedMP & 84.13 $\pm$ 0.16 & 47.09 $\pm$ 0.33 & 17.63 $\pm$ 0.23 & 82.18 $\pm$ 0.23 & 33.45 $\pm$ 0.44 & 20.61 $\pm$ 0.33 & 83.97 $\pm$ 0.30 & 75.55 $\pm$ 0.23\\
ScaleFL & 83.57 $\pm$ 0.23 & 43.98 $\pm$ 0.26 & 20.10 $\pm$ 0.45 & 81.13 $\pm$ 0.18 & 30.85 $\pm$ 0.29 & 21.23 $\pm$ 0.35 & 83.88 $\pm$ 0.22 & 70.11 $\pm$ 0.41 \\
Flado & 83.78 $\pm$ 0.15 & 42.25 $\pm$ 0.30 & 18.82 $\pm$ 0.19 & 81.01 $\pm$ 0.19 & 29.32 $\pm$ 0.35 & 19.76 $\pm$ 0.19 & 82.54 $\pm$ 0.28 & 71.17 $\pm$ 0.32 \\
DepthFL & 84.20 $\pm$ 0.24 & 47.28 $\pm$ 0.44 & 20.02 $\pm$ 0.34 & 81.33 $\pm$ 0.26 & 31.13 $\pm$ 0.37 & 20.17 $\pm$ 0.31 & 85.10 $\pm$ 0.28 & 73.43 $\pm$ 0.31 \\
FedDSE & 85.02 $\pm$ 0.13 & 46.23 $\pm$ 0.49 & 19.04 $\pm$ 0.36 & 81.98 $\pm$ 0.27 & 31.77 $\pm$ 0.51 & 20.25 $\pm$ 0.28 & 84.01 $\pm$ 0.31 & 70.88 $\pm$ 0.23 \\
AdaptiveFL & 83.87 $\pm$ 0.17 & 43.22 $\pm$ 0.46 & 19.55 $\pm$ 0.41 & 80.79 $\pm$ 0.23 & 30.16 $\pm$ 0.22 & 19.76 $\pm$ 0.33 & 84.15 $\pm$ 0.27 & 69.83 $\pm$ 0.17 \\
FlexFL & 85.13 $\pm$ 0.11 & 49.21 $\pm$ 0.45 & 22.23 $\pm$ 0.32 & 83.10 $\pm$ 0.25 & 35.27 $\pm$ 0.33 & 22.19 $\pm$ 0.44 & 86.02 $\pm$ 0.19 & 76.24 $\pm$ 0.56 \\
\midrule
\textbf{SubFLOT} & \textbf{86.89 $\pm$ 0.14} & \textbf{58.37 $\pm$ 0.24} & \textbf{29.30 $\pm$ 0.13} & \textbf{83.78 $\pm$ 0.16} & \textbf{44.88 $\pm$ 0.24} & \textbf{25.15 $\pm$ 0.24} & \textbf{87.88 $\pm$ 0.16} & \textbf{79.72 $\pm$ 0.21}\\
\bottomrule
\end{tabular}
}
\end{table*}

\begin{table*}[ht]
\centering
\vspace{-1mm}
\caption{Experimental results(\%) under feature shift Settings.}
\vspace{-1mm}
\label{table:result_feature_shift}
\resizebox{0.9\linewidth}{!}{
\begin{tabular}{l|ccccc|c|cccc|c}
\toprule
Datasets & \multicolumn{6}{c|}{Digit5} & \multicolumn{5}{c}{PACS} \\
\midrule
Domains   & Mnistm & Mnist & Syn & Usps & Svhn  & Avg. & Photo & Art Painting & Cartoon &  Sketch & Avg.    \\ 
\midrule
HeteroFL & 67.44 & 94.68 & 71.73 & 94.49 & 59.60 & 77.59 & 16.23 & 13.66 & 20.27 & 26.90 & 19.27\\
FedRolex & 66.07 & 94.30 & 70.60 & 94.20 & 52.70 & 75.57 & 9.30 & 16.76 & 15.78 & 28.30 & 17.53\\ 
FedDrop & 50.14 & 84.14 & 53.48 & 83.75 & 40.45 & 62.39 & 17.43 & 11.86 & 14.26 & 25.76 & 17.33\\
FedMP & 64.69 & 93.86 & 67.97 & 93.32 & 54.91 & 74.95 & 9.30 & 10.72 & 14.71 & 22.59 & 14.33\\
ScaleFL & 65.23  & 92.72  & 68.45  & 93.01  & 55.89 & 75.06 & 15.67  & 14.23  & 20.34  & 26.83 & 19.27 \\
Flado & 68.56  & 94.89  & 72.34  & 94.78  & 60.45 & 78.20 & 16.98  & 13.12  & 19.87  & 26.45 & 19.11 \\
DepthFL & 69.87  & 95.01  & 73.12  & 94.90  & 61.23 & 78.83 & 16.01  & 14.56  & 20.98  & 27.23 & 19.70 \\
FedDSE & 67.90  & 94.76  & 72.89  & 94.67  & 59.34 & 77.91 & 15.45  & 13.98  & 20.56  & 26.78 & 19.19 \\
AdaptiveFL & 68.34  & 94.90  & 73.23  & 94.87 & 60.12 & 78.29 & 15.78  & 14.12  & 21.01  & 27.34  & 19.56 \\
FlexFL & 69.45  & 95.23  & 74.12  & 95.10  & 62.34 & 79.25 & 16.34  & 14.78  & 21.56  & 28.01  & 20.17 \\
\midrule
\textbf{SubFLOT} & \textbf{88.81} & \textbf{98.65} & \textbf{92.27} & \textbf{98.20} & \textbf{84.98} & \textbf{92.58} &\textbf{ 48.23} & \textbf{28.73} & \textbf{42.55}& \textbf{46.83} & \textbf{41.58}\\
\bottomrule
\end{tabular}
}
\vspace{-3mm}
\end{table*}
\paragraph{Baselines and Metrics}
We benchmark SubFLOT against nine state-of-the-art federated pruning methods: HeteroFL~\cite{Diao2020HeteroFLCA}, FedRolex~\cite{alam2022fedrolex}, FedDrop~\cite{caldas2019feddrop}, FedMP~\cite{jiang2022fedmp}, ScaleFL~\cite{ilhan2023scalefl}, Flado~\cite{liao2023adaptive}, DepthFL~\cite{kim2022depthfl}, FedDSE~\cite{wang2024feddse}, AdaptiveFL~\cite{jia2024adaptivefl} and FlexFL~\cite{chen2024flexfl}. 
The primary evaluation metric is the mean accuracy of the local models, averaged over three independent runs with different random seeds.

\paragraph{Implementation Details}
All experiments are implemented in PyTorch and conducted on a server equipped with 4 NVIDIA RTX 4090 GPUs. Unless stated otherwise, we simulate a system with $N=20$ clients, all of which participate in each round (join ratio $p=1.0$). The training proceeds for 200 communication rounds. In each round, clients perform $E=5$ local epochs using Stochastic Gradient Descent (SGD) with a learning rate of $lr=0.001$ and a batch size of $B=256$. We define a set of four available pruning rates, \{0, 1/4, 1/2, 3/4\}, from which each client randomly samples to determine its local model capacity. For our SubFLOT, the fusion hyperparameter $\alpha$ (Eq.~\ref{eq:otp_mix}) is set to 0.5, and the regularization weight $\lambda$ (Eq.~\ref{eq:final_loss}) is set to 1.0.

\subsection{Performance Evaluation}

\paragraph{Label Skew Setting}
We first measured the performance of SubFLOT against baselines on datasets with label shifts and the experimental results are summarized in Table \ref{table:result_label_shift}. The data clearly shows that SubFLOT consistently outperforms state-of-the-art algorithms by a substantial margin across all datasets, thereby affirming its superior capability in handling heterogeneous data distributions. 
Notably, the performance advantage of SubFLOT becomes increasingly pronounced as the task complexity escalates. 
Furthermore, the exceptional performance of SubFLOT is not confined to computer vision; its effectiveness on the AG News dataset demonstrates its strong generalization capability across different data modalities.

\paragraph{Feature Shift Setting}
To evaluate the performance of SubFLOT in the presence of feature shifts, we comprehensively measured the accuracy of SubFLOT on each domain as well as the average accuracy across the entire dataset. 
As shown in Table \ref{table:result_feature_shift}, SubFLOT achieves significant improvements in accuracy on each domain. This not only demonstrates that the OTP module can effectively enable the personalization of local models but also highlights the joint contribution of the SAR and OTA modules in mitigating domain shifts, thereby enabling the model to learn more stable representations.

\paragraph{Real-world Setting}
The experiments on the HAR dataset, which simulates a practical real-world environment, offer crucial insights. As detailed in Table \ref{table:result_label_shift}, there is a distinct performance advantage for methods that utilize dynamic adaptation strategies like FedMP(magnitude) and FlexFL(APoZ), significantly surpassing statically configured models such as HeteroFL. This finding underscores the necessity of dynamically tailoring model architectures to achieve personalization in real-world federated scenarios.
SubFLOT is architected on this very principle. By pioneering the use of Optimal Transport for both personalized submodel generation (OTP) and aggregation (OTA), our framework excels in this realistic setting.

\subsection{Feature Visualization}
\begin{figure*}[ht]
\centering
\includegraphics[width=1\linewidth]{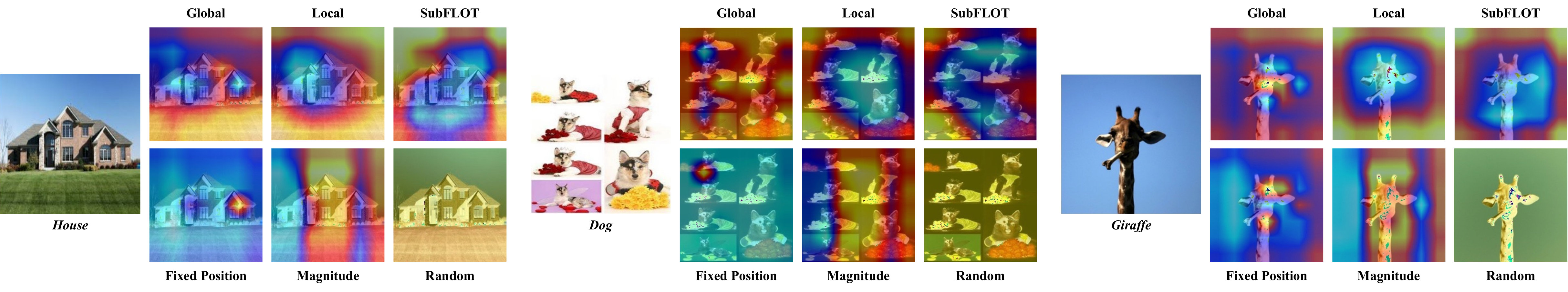} 
\caption{Feature Visualization of our SubFLOT.}
\label{fig:visualization}
\vspace{-1mm}
\end{figure*}
To intuitively understand how our method preserves client-specific knowledge, we use Grad-CAM~\cite{selvaraju2017grad} to visualize feature attention maps on the PACS dataset (Figure~\ref{fig:visualization}). We compare the submodel generated by SubFLOT against those from three baseline pruning strategies (fixed-position, magnitude-based, and random), using the client's historical model as a reference for its local feature priorities.

The visualization reveals a striking correspondence between the attention map of the SubFLOT-generated submodel and the client's historical model. Both highlight the same semantically critical regions, empirically validating that our OTP module successfully preserves and adapts task-relevant features. In contrast, the baseline methods fall short: fixed-position pruning fails to adapt from the global model's focus; magnitude-based pruning provides only a coarse approximation of local features; and random pruning results in diffuse, incoherent attention. This qualitative analysis demonstrates SubFLOT's unique ability to generate submodels that are not only compact but also semantically aligned with each client's data distribution.

\subsection{Scalability and Generalization}
\paragraph{Impact of Number of Clients}
We test SubFLOT's scalability on CIFAR100 ($\beta=0.1$) by varying the number of clients (10, 30, 50) under full participation, and with client pools of 50 and 100 under 10\% partial participation. 
The comprehensive results, presented in Table \ref{table:num_of_clients}, unequivocally demonstrate the superior scalability of SubFLOT. In every experimental setup, SubFLOT not only outperforms all baseline methods but also exhibits remarkable performance consistency. Notably, we observe that as the number of clients grows, the challenge of learning effective personalized models intensifies, causing a pronounced drop in the accuracy of most baselines. SubFLOT, however, proves resilient to this effect, sustaining a high level of performance with minimal degradation. This resilience validates the scalability of our approach, affirming its capability to maintain accuracy and efficiency in large-scale federated systems with a large and diverse client base.
\begin{table}[ht]
\centering
\vspace{-1mm}
\caption{Impact of Number of Clients.}
\vspace{-1mm}
\resizebox{\linewidth}{!}{
\begin{tabular}{l|ccc|cc}
\toprule
   & \multicolumn{3}{c|}{$p = 1$} & \multicolumn{2}{c}{$p = 0.1$}\\
\midrule
Clients & $N = 10$ & $N = 30$ & $N = 50$ & $N = 50$ & $N = 100$ \\ 
\midrule
HeteroFL & 24.73 & 18.91 & 18.79 & 16.28 & 14.60 \\
FedRolex & 29.16 & 19.40 & 19.45 & 17.05 & 15.21 \\ 
FedDrop & 12.09 & 9.05 & 10.36 & 9.48 & 8.97 \\
FedMP & 26.72 & 18.33 & 17.92 & 16.32 & 14.33 \\
ScaleFL & 24.84 & 18.91 & 18.82 & 16.28 & 15.02\\
Flado & 24.56 & 18.49 & 18.23 & 15.76 & 14.27 \\
DepthFL & 26.11 & 19.52 & 19.16 & 17.37 & 15.36\\
FedDSE & 25.23 & 19.13 & 18.91 & 16.84 & 14.79 \\
AdaptiveFL & 24.28 & 18.47 & 18.52 & 16.32 & 14.64\\
FlexFL & 28.97 & 20.12 & 20.24 & 18.87 & 16.26 \\
\midrule
\textbf{SubFLOT} & \textbf{39.20} & \textbf{38.65} & \textbf{38.50} & \textbf{35.72} & \textbf{32.61} \\
\bottomrule
\end{tabular}
}
\vspace{-3mm}
\label{table:num_of_clients}
\end{table}
\paragraph{Impact of Heterogeneity Degree}
To assess robustness against varying non-IID levels, we adjust the number of classes per client (pathological skew) and the Dirichlet parameter $\beta$ (practical skew) on CIFAR-10.
The results are illustrated in Table \ref{table:hetero_degree}.
In the pathological setting, as the number of classes per client increases, the amount of training data available for each class decreases due to the fixed total size of the CIFAR10 dataset. Similarly, in the practical setting, a larger $\beta$ value in the Dirichlet distribution leads to a more uniform distribution of data across clients, which also reduces the quantity of samples per class for individual clients. Consequently, as the degree of heterogeneity decreases, the performance of all algorithms exhibits a slight decline.
Nevertheless, despite these challenges, our proposed SubFLOT consistently outperforms other algorithms across all scenarios. 
\begin{table}[ht]
\centering
\vspace{-1mm}
\caption{Impact of Heterogeneity Degree.}
\vspace{-1mm}
\resizebox{\linewidth}{!}{
\begin{tabular}{l|ccc|ccc}
\toprule
 & \multicolumn{3}{c|}{Pathological Settings} & \multicolumn{3}{c}{Practical Settings} \\ 
\midrule
Params. & $n = 4$ & $n = 6$ & $n = 8$ & $\beta = 0.3$ & $\beta = 0.5$ & $\beta = 1$ \\
\midrule
HeteroFL & 68.85 & 59.80 & 54.63 & 75.60 & 65.45 & 60.36\\
FedRolex & 66.25 & 56.68 & 49.97 & 73.88 & 63.55 & 56.91\\ 
FedDrop & 61.30 & 50.11 & 43.28 & 70.06 & 59.92 & 52.40\\
FedMP & 69.47 & 61.06 & 54.41 & 76.14 & 66.92 & 61.57 \\
ScaleFL & 63.21 & 53.45 & 47.89 & 71.02 & 60.78 & 54.32 \\
Flado & 64.89 & 54.98 & 49.12 & 73.21 & 63.01 & 56.78 \\
DepthFL & 65.45 & 55.67 & 48.90 & 73.34 & 63.10 & 58.67 \\
FedDSE & 67.12 & 57.34 & 53.23 & 75.56 & 64.23 & 59.72 \\
AdaptiveFL & 66.78 & 56.90 & 50.45 & 74.89 & 63.56 & 58.45 \\
FlexFL & 68.56 & 61.78 & 54.90 & 76.45 & 67.23 & 62.87 \\
\midrule
\textbf{SubFLOT} & \textbf{73.58} & \textbf{66.02} & \textbf{62.08} & \textbf{78.92} & \textbf{70.71} & \textbf{65.04}\\
\bottomrule
\end{tabular}
}
\vspace{-5mm}
\label{table:hetero_degree}
\end{table}
\paragraph{Impact of Model Sparsity}
We evaluate SubFLOT's robustness to model sparsity on CIFAR10($\beta=0.1$) across three static regimes with increasing average pruning rates (Low: \{0, 0.125, 0.25, 0.375\}, Medium: \{0, 0.25, 0.5, 0.75\}, High: \{0.25, 0.45, 0.65, 0.85\}) and a Dynamic setting where clients periodically re-sampled pruning rates from the Medium-level set during training. 
As presented in Table \ref{table:sparsity_degree}, an increased sparsity level predictably leads to a performance decline across all algorithms. However, SubFLOT demonstrates exceptional robustness, consistently outperforming all baselines under each sparsity regime, including the dynamic one. This unequivocally underscores the robustness of SubFLOT across a wide spectrum of model sparsities.
\begin{table}[ht]
\centering
\vspace{-1mm}
\caption{Impact of Model Sparsity.}
\vspace{-1mm}
\resizebox{0.87\linewidth}{!}{
\begin{tabular}{l|ccc|c}
\toprule
 & Low & Medium & High & Dynamic\\ 
\midrule
HeteroFL & 82.02 & 81.21 & 81.39 & 81.16\\
FedRolex & 82.21 & 81.04 & 80.02 & 81.22\\ 
FedDrop & 81.61 & 78.49 & 76.17 & 77.91\\
FedMP & 83.00 & 82.18 & 82.37 & 82.30 \\
ScaleFL & 82.17 & 81.13 & 81.13 & 81.06\\
Flado & 81.89 & 81.01 & 80.78 & 80.88 \\
DepthFL & 82.24 & 81.33 & 81.41 & 81.21\\
FedDSE & 83.16 & 81.98 & 81.54 & 81.72\\
AdaptiveFL & 81.87 & 80.79 & 80.62 & 81.46\\
FlexFL & 83.65 & 82.37 & 81.76 & 82.29\\
\midrule
\textbf{SubFLOT} & \textbf{84.27} & \textbf{83.78} & \textbf{82.38} & \textbf{83.54} \\
\bottomrule
\end{tabular}
}
\vspace{-3mm}
\label{table:sparsity_degree}
\end{table}

\subsection{Resource Efficiency}
SubFLOT is designed for high resource efficiency, addressing the prohibitive overheads of standard federated learning. On the server, the computational cost is primarily from the OTP module, which maintains a low theoretical complexity of $\mathcal{O}(L \cdot M^2)$, where $L$ is the number of layers and $M$ is the number of filters. This computational tractability is a direct result of our progressive layer-wise strategy and the use of highly optimized OT solvers~\cite{singh2020model}. Empirically, this translates to minimal latency; for instance, the per-client OTP operation for VGG11, MobileNetV2, and ResNet18 on an NVIDIA 4090 GPU requires just 0.12s, 1.15s, and 0.20s, respectively. Besides, by dispatching customized submodels, SubFLOT drastically alleviates the resource burden on edge devices. As presented in Table~\ref{tab:cost_comparison}, our method achieves a consistent reduction of over 50\% in both per-client, per-round communication and computation costs across all tested architectures when compared to FedAvg, underscoring its profound advantage in resource-constrained settings.
\begin{table}[ht]
\centering
\vspace{-1mm}
\caption{Resource Efficiency.}
\vspace{-1mm}
\label{tab:cost_comparison}
\resizebox{\linewidth}{!}{
\begin{tabular}{llccc}
\toprule
Metric & Method & VGG11 & MobileNet & ResNet18 \\
\midrule
\multirow{2}{*}{Comm.(MB)} & SubFLOT& 4.33 & 1.07 & 5.24 \\
                             & FedAvg & 9.23 & 2.24 & 11.18 \\
\midrule
\multirow{2}{*}{\makecell{Comp.\\(MFLOPs)}} & SubFLOT & 72.35 & 3.22 & 17.83 \\
                               & FedAvg & 153.50 & 6.52 & 37.18 \\
\bottomrule
\end{tabular}
}
\vspace{-3mm}
\end{table}

\subsection{Ablation Study}

To systematically evaluate the contributions of each component in SubFLOT, we conduct rigorous ablation experiments under practical federated learning scenarios using CIFAR10 ($\beta = 0.1$) and Digit5 datasets. The study involves three variants of our framework: 1) w/o OTP, where the OTP module is replaced with fixed-position pruning strategy; 2) w/o SAR, which removes the SAR term during local training, and 3) w/o OTA, where the OTA module is substituted with position-based aggregation using the initial submodel structures.

The experimental results, summarized in Table \ref{table:ablation}, reveal that the absence of any core component leads to significant performance degradation across all testbeds, confirming their synergistic roles in addressing parameter-space and feature-space heterogeneity. Notably, the removal of OTA causes the most pronounced accuracy drop. This underscores the indispensable function of OTA as a post-hoc alignment mechanism, which is essential for effectively aggregating the structurally diverse submodels produced by the OTP module.
\begin{table}[ht]
\centering
\vspace{-1mm}
\caption{Ablation Study.}
\vspace{-1mm}
\resizebox{0.9\linewidth}{!}{
\begin{tabular}{l|cccc}
\toprule
 & \textbf{SubFLOT} & w/o OTP & w/o SAR & w/o OTA  \\
\midrule
Cifar10 & \textbf{83.78} & 83.21 & 82.77 & 81.41  \\
Digit5 & \textbf{92.58} & 90.37 & 89.73 & 78.74  \\
\bottomrule
\end{tabular}
}
\label{table:ablation}
\vspace{-3mm}
\end{table}
\subsection{Hyperparameter Study}
To evaluate the impact of $\alpha$ in OTP/OTA module and $\lambda$ in the SLR module, we conducted extensive experiments on CIFAR10 and the results are presented in Table \ref{table:hyperparamter}.
First, we analyze the impact of $\alpha$, which balances the influence of the current global model against the client's local knowledge. The results indicate that increasing $\alpha$ generally leads to improved final accuracy. This is attributable to the fact that a larger $\alpha$ prioritizes the integration of the most up-to-date global information, thereby accelerating the model's convergence. However, we empirically observed that smaller $\alpha$ values yield more stable training dynamics. 
This reveals a trade-off between convergence speed and training stability in the selection of $\alpha$.

Next, we examine the effect of the regularization coefficient $\lambda$. Initially, performance improves as $\lambda$ increases, because a moderate regularization penalty effectively constrains local sub-models from deviating excessively from the global objective. This ensures that knowledge learned by each client remains consistent and aggregatable. Conversely, when $\lambda$ becomes overly large, performance begins to degrade. This decline is due to excessive regularization, where the stringent penalty stifles the model's flexibility and capacity to adapt to client-specific data.
\begin{table}[ht]
\centering
\vspace{-1mm}
\caption{Hyperparameter Study.}
\vspace{-1mm}
\label{table:hyperparamter}
\resizebox{\linewidth}{!}{
\begin{tabular}{l|cccc|cccc}
\toprule
 & \multicolumn{4}{c|}{Different $\alpha$} & \multicolumn{4}{c}{Different $\lambda$} \\
\midrule
HP. & 0.1 & 0.3 & 0.5 & 0.7 & 0.5 & 1 & 5 & 10 \\
\midrule
Acc. & 82.93 & 82.90 & 83.78 & \textbf{84.39} & 83.22 & 83.78 & \textbf{83.91} & 82.93 \\
\bottomrule
\end{tabular}
}
\vspace{-5mm}
\end{table}
\section{Conclusion}
\label{sec:conclusion}
This paper presents SubFLOT, a novel federated learning framework that fundamentally addresses the dual challenges of system and statistical heterogeneity through synergistic integration of optimal transport and adaptive regularization mechanisms. Through extensive experiments across varied scenarios and system configurations, we demonstrate that SubFLOT achieves superior performance compared to existing methods.
\section{Acknowledgement}
\label{sec:acknowledgement}
This work is supported by the National Key Research and Development Program of China No.2022ZD0115903. The work is also supported by Key Laboratory of Data Intelligence, Beijing.
{
    \small
    \bibliographystyle{ieeenat_fullname}
    \bibliography{main}

@String(CVPR= {IEEE Conf. Comput. Vis. Pattern Recog.})

@String(ECCV= {Eur. Conf. Comput. Vis.})

@String(ICASSP=	{ICASSP})

@String(ICIP = {IEEE Int. Conf. Image Process.})

@String(CVPR  = {CVPR})

@String(ECCV  = {ECCV})

@String(ICIP  = {ICIP})

@inproceedings{mcmahan2017communication,
  title={Communication-efficient learning of deep networks from decentralized data},
  author={McMahan, Brendan and Moore, Eider and Ramage, Daniel and Hampson, Seth and y Arcas, Blaise Aguera},
  booktitle={Artificial intelligence and statistics},
  pages={1273--1282},
  year={2017},
  organization={PMLR}
}

@inproceedings{yang2023efficient,
  title={Efficient model personalization in federated learning via client-specific prompt generation},
  author={Yang, Fu-En and Wang, Chien-Yi and Wang, Yu-Chiang Frank},
  booktitle={Proceedings of the IEEE/CVF International Conference on Computer Vision},
  pages={19159--19168},
  year={2023}
}

@article{frankle2018lottery,
  title={The lottery ticket hypothesis: Finding sparse, trainable neural networks},
  author={Frankle, Jonathan and Carbin, Michael},
  journal={arXiv preprint arXiv:1803.03635},
  year={2018}
}

@inproceedings{deng2022tailorfl,
  title={TailorFL: Dual-personalized federated learning under system and data heterogeneity},
  author={Deng, Yongheng and Chen, Weining and Ren, Ju and Lyu, Feng and Liu, Yang and Liu, Yunxin and Zhang, Yaoxue},
  booktitle={Proceedings of the 20th ACM conference on embedded networked sensor systems},
  pages={592--606},
  year={2022}
}

@inproceedings{li2019convergence,
  title={On the Convergence of FedAvg on Non-IID Data},
  author={Li, Xiang and Huang, Kaixuan and Yang, Wenhao and Wang, Shusen and Zhang, Zhihua},
  booktitle={International Conference on Learning Representations},
  year={2019}
}

@inproceedings{yi2024fedp3,
  title={FedP3: Federated Personalized and Privacy-friendly Network Pruning under Model Heterogeneity},
  author={Yi, Kai and Gazagnadou, Nidham and Richt{\'a}rik, Peter and Lyu, Lingjuan},
  booktitle={International Conference on Learning Representations},
  year={2024}
}

@inproceedings{li2021fedmask,
  title={Fedmask: Joint computation and communication-efficient personalized federated learning via heterogeneous masking},
  author={Li, Ang and Sun, Jingwei and Zeng, Xiao and Zhang, Mi and Li, Hai and Chen, Yiran},
  booktitle={Proceedings of the 19th ACM conference on embedded networked sensor systems},
  pages={42--55},
  year={2021}
}

@inproceedings{li2024global,
  title={Global and local prompts cooperation via optimal transport for federated learning},
  author={Li, Hongxia and Huang, Wei and Wang, Jingya and Shi, Ye},
  booktitle={Proceedings of the IEEE/CVF Conference on Computer Vision and Pattern Recognition},
  pages={12151--12161},
  year={2024}
}

@article{singh2020model,
  title={Model fusion via optimal transport},
  author={Singh, Sidak Pal and Jaggi, Martin},
  journal={Advances in Neural Information Processing Systems},
  volume={33},
  pages={22045--22055},
  year={2020}
}

@article{ek2024fedali,
  title={Fedali: Personalized federated learning with aligned prototypes through optimal transport},
  author={Ek, Sannara and Wang, Kaile and Portet, Fran{\c{c}}ois and Lalanda, Philippe and Cao, Jiannong},
  journal={arXiv preprint arXiv:2411.10595},
  year={2024}
}

@article{kantorovich2006translocation,
  title={On the translocation of masses},
  author={Kantorovich, Leonid V},
  journal={Journal of mathematical sciences},
  volume={133},
  number={4},
  pages={1381--1382},
  year={2006},
  publisher={Kluwer Academic Publishers-Consultants Bureau New York}
}

@inproceedings{damodaran2018deepjdot,
  title={{Deepjdot: Deep joint distribution optimal transport for unsupervised domain adaptation}},
  author={Damodaran, Bharath Bhushan and Kellenberger, Benjamin and Flamary, R{\'e}mi and Tuia, Devis and Courty, Nicolas},
  booktitle={Proceedings of the European conference on computer vision (ECCV)},
  pages={447--463},
  year={2018}
}

@inproceedings{fatras2021unbalanced,
  title={Unbalanced minibatch optimal transport; applications to domain adaptation},
  author={Fatras, Kilian and S{\'e}journ{\'e}, Thibault and Flamary, R{\'e}mi and Courty, Nicolas},
  booktitle={International Conference on Machine Learning},
  pages={3186--3197},
  year={2021},
  organization={PMLR}
}

@inproceedings{Diao2020HeteroFLCA,
  title = {HeteroFL: Computation and Communication Efficient Federated Learning for Heterogeneous Clients},
  author = {Enmao Diao and Jie Ding and Vahid Tarokh},
  booktitle = {International Conference on Learning Representations},
  year = {2021}
}

@article{alam2022fedrolex,
  title={Fedrolex: Model-heterogeneous federated learning with rolling sub-model extraction},
  author={Alam, Samiul and Liu, Luyang and Yan, Ming and Zhang, Mi},
  journal={Advances in neural information processing systems},
  volume={35},
  pages={29677--29690},
  year={2022}
}

@misc{caldas2019feddrop,
      title={Expanding the Reach of Federated Learning by Reducing Client Resource Requirements}, 
      author={Sebastian Caldas and Jakub Konečny and H. Brendan McMahan and Ameet Talwalkar},
      year={2019},
      eprint={1812.07210},
      archivePrefix={arXiv},
      primaryClass={cs.LG},
      url={https://arxiv.org/abs/1812.07210}, 
}

@inproceedings{ilhan2023scalefl,
  title={Scalefl: Resource-adaptive federated learning with heterogeneous clients},
  author={Ilhan, Fatih and Su, Gong and Liu, Ling},
  booktitle={Proceedings of the IEEE/CVF Conference on Computer Vision and Pattern Recognition},
  pages={24532--24541},
  year={2023}
}

@inproceedings{liao2023adaptive,
  title={Adaptive channel sparsity for federated learning under system heterogeneity},
  author={Liao, Dongping and Gao, Xitong and Zhao, Yiren and Xu, Cheng-Zhong},
  booktitle={Proceedings of the IEEE/CVF Conference on Computer Vision and Pattern Recognition},
  pages={20432--20441},
  year={2023}
}

@inproceedings{jiang2022fedmp,
  title={Fedmp: Federated learning through adaptive model pruning in heterogeneous edge computing},
  author={Jiang, Zhida and Xu, Yang and Xu, Hongli and Wang, Zhiyuan and Qiao, Chunming and Zhao, Yangming},
  booktitle={2022 IEEE 38th International Conference on Data Engineering (ICDE)},
  pages={767--779},
  year={2022},
  organization={IEEE}
}

@inproceedings{wang2024feddse,
  title={Feddse: Distribution-aware sub-model extraction for federated learning over resource-constrained devices},
  author={Wang, Haozhao and Jia, Yabo and Zhang, Meng and Hu, Qinghao and Ren, Hao and Sun, Peng and Wen, Yonggang and Zhang, Tianwei},
  booktitle={Proceedings of the ACM Web Conference 2024},
  pages={2902--2913},
  year={2024}
}

@inproceedings{kim2022depthfl,
  title={Depthfl: Depthwise federated learning for heterogeneous clients},
  author={Kim, Minjae and Yu, Sangyoon and Kim, Suhyun and Moon, Soo-Mook},
  booktitle={The Eleventh International Conference on Learning Representations},
  year={2022}
}

@article{chen2024flexfl,
  title={Flexfl: Heterogeneous federated learning via apoz-guided flexible pruning in uncertain scenarios},
  author={Chen, Zekai and Jia, Chentao and Hu, Ming and Xie, Xiaofei and Li, Anran and Chen, Mingsong},
  journal={IEEE Transactions on Computer-Aided Design of Integrated Circuits and Systems},
  volume={43},
  number={11},
  pages={4069--4080},
  year={2024},
  publisher={IEEE}
}

@inproceedings{jia2024adaptivefl,
  title={AdaptiveFL: Adaptive heterogeneous federated learning for resource-constrained AIoT systems},
  author={Jia, Chentao and Hu, Ming and Chen, Zekai and Yang, Yanxin and Xie, Xiaofei and Liu, Yang and Chen, Mingsong},
  booktitle={Proceedings of the 61st ACM/IEEE Design Automation Conference},
  pages={1--6},
  year={2024}
}

@article{krizhevsky2009learning,
  title={{Learning Multiple Layers of Features From Tiny Images}},
  author={Krizhevsky, Alex and Geoffrey, Hinton},
  year={2009},
  journal={Technical Report}
}

@article{chrabaszcz2017downsampled,
  title={{A Downsampled Variant of Imagenet as an Alternative to the Cifar Datasets}},
  author={Chrabaszcz, Patryk and Loshchilov, Ilya and Hutter, Frank},
  journal={arXiv preprint arXiv:1707.08819},
  year={2017}
}

@article{zhang2015character,
  title={Character-level convolutional networks for text classification},
  author={Zhang, Xiang and Zhao, Junbo and LeCun, Yann},
  journal={Advances in neural information processing systems},
  volume={28},
  year={2015}
}

@article{lecun1998gradient,
  title={Gradient-based Learning Applied to Document Recognition},
  author={LeCun, Yann and Bottou, L{\'e}on and Bengio, Yoshua and Haffner, Patrick},
  journal={Proceedings of the IEEE},
  volume={86},
  number={11},
  pages={2278--2324},
  year={1998},
  publisher={Ieee}
}

@inproceedings{yu2022pacs,
  title={Pacs: A dataset for physical audiovisual commonsense reasoning},
  author={Yu, Samuel and Wu, Peter and Liang, Paul Pu and Salakhutdinov, Ruslan and Morency, Louis-Philippe},
  booktitle={European Conference on Computer Vision},
  pages={292--309},
  year={2022},
  organization={Springer}
}

@inproceedings{anguita2012human,
  title={Human activity recognition on smartphones using a multiclass hardware-friendly support vector machine},
  author={Anguita, Davide and Ghio, Alessandro and Oneto, Luca and Parra, Xavier and Reyes-Ortiz, Jorge L},
  booktitle={Ambient Assisted Living and Home Care: 4th International Workshop, IWAAL 2012, Vitoria-Gasteiz, Spain, December 3-5, 2012. Proceedings 4},
  pages={216--223},
  year={2012},
  organization={Springer}
}

@article{simonyan2014very,
  title={Very deep convolutional networks for large-scale image recognition},
  author={Simonyan, Karen and Zisserman, Andrew},
  journal={arXiv preprint arXiv:1409.1556},
  year={2014}
}

@inproceedings{sandler2018mobilenetv2,
  title={Mobilenetv2: Inverted residuals and linear bottlenecks},
  author={Sandler, Mark and Howard, Andrew and Zhu, Menglong and Zhmoginov, Andrey and Chen, Liang-Chieh},
  booktitle={Proceedings of the IEEE conference on computer vision and pattern recognition},
  pages={4510--4520},
  year={2018}
}

@InProceedings{pmlr-v139-feng21f,
  title = 	 {KD3A: Unsupervised Multi-Source Decentralized Domain Adaptation via Knowledge Distillation},
  author =       {Feng, Haozhe and You, Zhaoyang and Chen, Minghao and Zhang, Tianye and Zhu, Minfeng and Wu, Fei and Wu, Chao and Chen, Wei},
  booktitle = 	 {Proceedings of the 38th International Conference on Machine Learning},
  pages = 	 {3274--3283},
  year = 	 {2021},
  editor = 	 {Meila, Marina and Zhang, Tong},
  volume = 	 {139},
  series = 	 {Proceedings of Machine Learning Research},
  month = 	 {18--24 Jul},
  publisher =    {PMLR}
}

@inproceedings{selvaraju2017grad,
  title={Grad-cam: Visual explanations from deep networks via gradient-based localization},
  author={Selvaraju, Ramprasaath R and Cogswell, Michael and Das, Abhishek and Vedantam, Ramakrishna and Parikh, Devi and Batra, Dhruv},
  booktitle={Proceedings of the IEEE international conference on computer vision},
  pages={618--626},
  year={2017}
}

@inproceedings{he2016deep,
  title={Deep Residual Learning for Image Recognition},
  author={He, Kaiming and Zhang, Xiangyu and Ren, Shaoqing and Sun, Jian},
  booktitle={{Proceedings of the IEEE/CVF Conference on Computer Vision and Pattern Recognition}},
  year={2016}
}

@article{joulinetal2017bag,
  title={Bag of tricks for efficient text classification},
  author={Joulin, Armand and Grave, Edouard and Bojanowski, Piotr and Mikolov, Tomas},
  journal={arXiv preprint arXiv:1607.01759},
  year={2016}
}

@inproceedings{zeng2014convolutional,
  title={Convolutional neural networks for human activity recognition using mobile sensors},
  author={Zeng, Ming and Nguyen, Le T and Yu, Bo and Mengshoel, Ole J and Zhu, Jiang and Wu, Pang and Zhang, Joy},
  booktitle={6th international conference on mobile computing, applications and services},
  pages={197--205},
  year={2014},
  organization={IEEE}
}

@InProceedings{Hao_2025_CVPR,
    author    = {Hao, Chenhe and Xie, Weiying and Li, Daixun and Qin, Haonan and Ye, Hangyu and Fang, Leyuan and Li, Yunsong},
    title     = {FedCS: Coreset Selection for Federated Learning},
    booktitle = {Proceedings of the IEEE/CVF Conference on Computer Vision and Pattern Recognition (CVPR)},
    month     = {June},
    year      = {2025},
    pages     = {15434-15443}
}

@inproceedings{Zhang2023fedcp,
  author = {Zhang, Jianqing and Hua, Yang and Wang, Hao and Song, Tao and Xue, Zhengui and Ma, Ruhui and Guan, Haibing},
  title = {FedCP: Separating Feature Information for Personalized Federated Learning via Conditional Policy},
  year = {2023},
  booktitle = {Proceedings of the 29th ACM SIGKDD Conference on Knowledge Discovery and Data Mining}
}

@inproceedings{zhang2023eliminating,
  title={Eliminating Domain Bias for Federated Learning in Representation Space},
  author={Jianqing Zhang and Yang Hua and Jian Cao and Hao Wang and Tao Song and Zhengui XUE and Ruhui Ma and Haibing Guan},
  booktitle={Thirty-seventh Conference on Neural Information Processing Systems},
  year={2023},
  url={https://openreview.net/forum?id=nO5i1XdUS0}
}

@misc{perfedavg,
      title={Personalized Federated Learning: A Meta-Learning Approach}, 
      author={Alireza Fallah and Aryan Mokhtari and Asuman Ozdaglar},
      year={2020},
      eprint={2002.07948},
      archivePrefix={arXiv},
      primaryClass={cs.LG},
      url={https://arxiv.org/abs/2002.07948}, 
}

@misc{pfedme,
      title={Personalized Federated Learning with Moreau Envelopes}, 
      author={Canh T. Dinh and Nguyen H. Tran and Tuan Dung Nguyen},
      year={2022},
      eprint={2006.08848},
      archivePrefix={arXiv},
      primaryClass={cs.LG},
      url={https://arxiv.org/abs/2006.08848}, 
}

@misc{fedamp,
      title={Personalized Cross-Silo Federated Learning on Non-IID Data}, 
      author={Yutao Huang and Lingyang Chu and Zirui Zhou and Lanjun Wang and Jiangchuan Liu and Jian Pei and Yong Zhang},
      year={2021},
      eprint={2007.03797},
      archivePrefix={arXiv},
      primaryClass={cs.LG},
      url={https://arxiv.org/abs/2007.03797}, 
}

@article{fedfomo,
  title={Personalized federated learning with first order model optimization},
  author={Zhang, Michael and Sapra, Karan and Fidler, Sanja and Yeung, Serena and Alvarez, Jose M},
  journal={arXiv preprint arXiv:2012.08565},
  year={2020}
}

@inproceedings{he2025dynfed,
  title={DynFed: Adaptive Federated Learning via Quantization-Aware Knowledge Distillation},
  author={He, Nan and Chen, Yiming and Jiang, Zheng and Yang, Song and Sun, Lifeng},
  booktitle={Proceedings of the 33rd ACM International Conference on Multimedia},
  pages={11844--11852},
  year={2025}
}

@inproceedings{chen2024fedawa,
  title={Fedawa: Aggregation weight adjustment in federated domain generalization},
  author={Chen, Yiming and He, Nan and Sun, Lifeng},
  booktitle={2024 IEEE International Conference on Image Processing (ICIP)},
  pages={451--457},
  year={2024},
  organization={IEEE}
}

@inproceedings{chen2025fedtg,
  title={FedTG: Text-guided Federated Domain Generalization},
  author={Chen, Yiming and He, Nan and Sun, Lifeng},
  booktitle={ICASSP 2025-2025 IEEE International Conference on Acoustics, Speech and Signal Processing (ICASSP)},
  pages={1--5},
  year={2025},
  organization={IEEE}
}

@inproceedings{jiangmedvr,
  title={MedVR: Annotation-Free Medical Visual Reasoning via Agentic Reinforcement Learning},
  author={Jiang, Zheng and Guo, Heng and Fang, Chengyu and Xiao, Changchen and Hu, Xinyang and Sun, Lifeng and Xu, Minfeng},
  booktitle={The Fourteenth International Conference on Learning Representations},
  year={2026}
}
}

\clearpage
\maketitlesupplementary

\section{Experimental Setup}
\label{suppl:exp}
\subsection{Datasets} 
Our evaluation utilizes a diverse suite of seven datasets to ensure a comprehensive assessment of our method's performance across various data modalities, including computer vision (CV), natural language processing (NLP), and Internet of Things (IoT) sensor data. The benchmark includes standard single-domain datasets such as CIFAR-10/100, Tiny-ImageNet, AG News, and HAR. To specifically evaluate robustness against feature distribution shifts, we also incorporate two multi-domain datasets: \textbf{Digit-5}, comprising five distinct handwritten digit domains (MNIST, MNIST-M, USPS, SVHN, and SYN), and \textbf{PACS}, which includes four artistic domains (Photo, Art, Cartoon, and Sketch). Key statistics for each dataset are summarized in Table~\ref{table:dataset_details}.

\begin{table*}[t]
\centering
\caption{Dataset Specifications. Our evaluation covers diverse data modalities and includes multi-domain benchmarks (Digit-5, PACS) to rigorously assess model generalization and robustness against feature distribution shifts.}
\label{table:dataset_details}
\resizebox{0.7\linewidth}{!}{
\begin{tabular}{lccccc}
\toprule
\textbf{Dataset} & \textbf{Classes} & \textbf{Training Samples} & \textbf{Test Samples} & \textbf{Domains} & \textbf{Modality} \\ 
\midrule
CIFAR-10      & 10  & 50,000  & 10,000  & 1 & CV \\
CIFAR-100     & 100 & 50,000  & 10,000  & 1 & CV \\ 
Tiny-ImageNet & 200 & 100,000 & 10,000  & 1 & CV \\ 
Digit-5       & 10  & 130,288 & 32,572  & 5 & CV \\
PACS          & 7   & 7,988   & 2,003   & 4 & CV \\
\midrule
AG News       & 4   & 120,000 & 7,600   & 1 & NLP \\
\midrule
HAR           & 6   & 7,352   & 2,947   & 1 & IoT Sensor \\
\bottomrule
\end{tabular}
}
\end{table*}

\subsection{Data Partition}
To simulate realistic Federated Learning (FL) environments, we construct three distinct statistical heterogeneity scenarios~\cite{chen2025fedtg,jiangmedvr}. 
For the \textit{label skew} scenario, we implement two common settings: the pathological setting and the practical setting~\cite{li2019convergence}. For the \textit{pathological label skew}, we sample data with label amount 2/10/20 for each client on Cifar10/Cifar100/TinyImageNet from a total of 10/100/200 categories. For the \textit{practical label skew}, we employ a Dirichlet distribution (default $\beta=0.1$) to generate realistic partially-overlapping class distributions for Cifar10, Cifar100, TinyImageNet and AG News.
For the \textit{feature shift} scenario, We utilize the Digit5 (5 domains) and PACS (4 domains) datasets. Each client participating in the FL system is assigned data from one of these distinct domains.
Finally, we use the HAR dataset to represent a \textit{real-world} scenario, which provides a natural partitioning of sensor data from 30 users performing six activities.

\section{Method Details}

Algorithm~\ref{algo:otp} delineates the procedural details of our Optimal Transport-based Pruning (OTP) module. The Optimal Transport-enhanced Aggregation (OTA) module then adapts this layer-wise mechanism, performing a conceptually inverse operation to map the updated client submodels back into the global parameter space for aggregation.

To accommodate modern network architectures, we incorporate specialized handling for specific layer types.
\begin{itemize}
    \item \textbf{Residual Blocks:} To maintain the continuity of the transport map across skip connections, the transport matrices derived from parallel branches are averaged prior to propagation to subsequent layers.
    \item \textbf{Batch Normalization Layers:} Since these layers perform channel-wise normalization without altering the dimensional permutation of the feature space, the incoming transport matrices are passed through without modification.
\end{itemize}
The complete end-to-end workflow of the SubFLOT framework is summarized in Algorithm~\ref{alg:SubFLOT}.

\begin{algorithm}[ht] 
\caption{Optimal Transport-based Parameter Alignment and Fusion (OTP)}
\label{algo:otp}
\KwIn{Global model $M_G$, client model $M_i$, number of layers $L$, fusion ratio $\alpha$}
\KwOut{Personalized fused model $\widetilde{M}_i$}

\For{each client $i \in \{1, 2, \dots, N\}$}{
    Initialize transport matrix for the input layer: $T^{(0)}_i = I$\;
    
    \For{each layer $l \in \{1, 2, \dots, L\}$}{
        $\widehat{W}_G^{(l,l-1)} = W_G^{(l,l-1)} T_i^{(l-1)}$\;
        
        $C^{(l)}_{jk} = \| \widehat{W}_G^{(l,l-1)}[j] - W_i^{(l,l-1)}[k] \|$\;
        
        $T^{(l)}_i = \arg\min_{T} \langle C^{(l)}, T \rangle_F \quad \text{s.t.} \quad T \mathbf{1}_{|\nu|} = \mu, \; T^\top \mathbf{1}_{|\mu|} = \nu$\;
        
        $T^{(l)}_i = T^{(l)}_i \odot \left( \frac{1}{\mathbf{1}_m^\top T^{(l)}_i} \right)$\;
        
        $\widetilde{W}_{\text{aligned}}^{(l,l-1)} = {T^{(l)}_i}^\top \widehat{W}_G^{(l,l-1)}$\;
        
        $\widetilde{W}_i^{(l,l-1)} = \alpha \cdot \widetilde{W}_{\text{aligned}}^{(l,l-1)} + (1 - \alpha) \cdot W_i^{(l,l-1)}$\;
    }
    Construct the personalized model $\widetilde{M}_i$ from the fused weights $\{\widetilde{W}_i^{(l,l-1)}\}_{l=1}^L$\;
}
\Return{Personalized models $\{\widetilde{M}_i\}_{i=1}^N$}\;
\end{algorithm}

\begin{algorithm}[ht]
\caption{SubFLOT: Federated Submodel Learning via Optimal Transport}
\label{alg:SubFLOT}
\KwIn{Communication rounds $T$, local epochs $R$, client number $N$, learning rate $\eta$, fusion ratio $\alpha$, regularization factor $\lambda$, local dataset $D_i$ and pruning ratio $\rho_i$ for each client $i$}
Initialize global model $W_G^0$ and client submodels $\{{W_i^0}\}_{i=1}^N$\;

\For{each communication round $t \in \{1, \dots, T\}$}{
    \For{each client $i \in \{1, \dots, N\}$ \textbf{in parallel}}{
        $\widetilde{W}_i^t \leftarrow \text{OTP}(W_G^{t-1}, W_i^{t-1}, \alpha)$ \quad \CommentSty{// See Algorithm~\ref{algo:otp}}\;
        $W_i^t \leftarrow \text{ClientUpdate}(i, \widetilde{W}_i^t, D_i, R, \eta)$\;
    }
    $W_G^t \leftarrow \text{OTA}(\{W_i^t\}_{i=1}^N)$\;
}

\vspace{1em}
\SetKwProg{Fn}{Procedure}{}{}
\Fn{\text{ClientUpdate($i$, $\widetilde{W}_i^t$, $D_i$, $R$, $\eta$)}}{
    \For{each local epoch $r \in \{1, \dots, R\}$}{
        \For{each batch $(x, y) \in D_i$}{
            Update weights of $\widetilde{W}_i^t$ using SGD: $\widetilde{W}_i^t \leftarrow \widetilde{W}_i^t - \eta \nabla \mathcal{L}_i(\widetilde{W}_i^t; x, y)$\;
        }
    }
    \Return{Updated local model $\widetilde{W}_i^t$}\;
}
\end{algorithm}

\section{Additional Experimental Results}
\label{add_exp}

\subsection{Server-Side Latency and Scalability Analysis}
\label{app:latency}

A major concern regarding our framework is whether the Optimal Transport (OT) computations introduced by SubFLOT incur prohibitive server-side latency. To address this concern, we conducted a wall-clock time analysis on CIFAR-10, as summarized in Fig.~\ref{fig:time_analysis}. We report two complementary metrics: 
(i) the total elapsed time required to complete a fixed number of communication rounds, and 
(ii) the total elapsed time required to reach a target test accuracy (80\% in our experiment).

The results indicate that the additional OT-related computation constitutes only a modest fraction of the total training time. More importantly, this overhead is compensated by the improved optimization efficiency of SubFLOT, which reaches a desirable accuracy level in fewer effective rounds and with a better overall time-to-accuracy trade-off. In other words, although OT introduces extra computation on the server, the resulting gains in convergence behavior make the full training process more efficient from a practical deployment perspective.

This observation is consistent with the intended design of SubFLOT. First, OT is executed on the server rather than on clients, which is desirable in federated settings because the server generally has substantially stronger computational resources than edge devices. Second, the OT computations for different clients are naturally parallelizable, since each client-specific transport plan can be computed independently. Therefore, in a realistic distributed infrastructure, the latency introduced by OT can be further amortized through parallel processing. Third, unlike client-side personalized pruning methods that require full-model training before pruning, our server-side personalization mechanism avoids imposing heavy burdens on resource-constrained clients, which is often a more critical bottleneck in practical FL systems.

Taken together, these results suggest that the overhead of OT is not a limiting factor for SubFLOT in realistic federated deployments. Instead, the method achieves a favorable balance between server-side computation and end-to-end training efficiency.

\begin{figure*}[ht]
  \centering
  \includegraphics[width=1.0\linewidth]{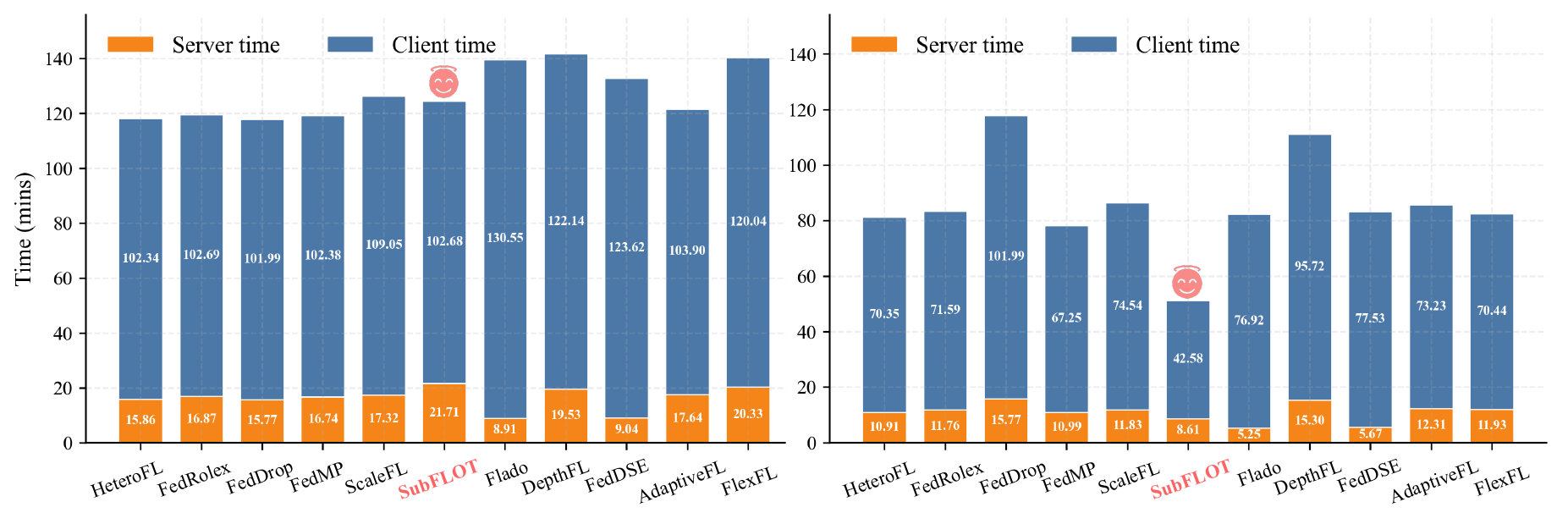}
  \caption{Comparison of total wall-clock time required to complete 200 communication rounds (left) and to reach 80\% test accuracy (right) on CIFAR-10. Although OT introduces additional server-side computation, SubFLOT achieves a superior time-to-accuracy trade-off due to faster convergence.}
  \label{fig:time_analysis}
\end{figure*}

\subsection{Hyperparameter Sensitivity and Practical Tuning Guidelines}
\label{app:hyper}

We further investigate the sensitivity of SubFLOT to its two key hyperparameters, namely the fusion coefficient $\alpha$ used in OTP and the adaptive regularization coefficient $\lambda$ used in SAR. The corresponding results on multiple datasets and under different non-IID settings are shown in Fig.~\ref{fig:hyper_analysis}.

\textbf{Effect of $\alpha$.}
The parameter $\alpha$ controls the extent to which the server-side personalized pruning process is influenced by the historical client model. Empirically, we observe that a relatively larger value of $\alpha$ can accelerate convergence in moderately heterogeneous scenarios, as it enables stronger client-specific adaptation. However, under severe statistical heterogeneity, excessive reliance on historical client information may amplify instability, especially when local distributions drift substantially from one another. In such cases, a smaller $\alpha$ often yields more stable optimization by preserving a stronger anchor to the global model. This phenomenon is especially evident on challenging benchmarks such as Digit5, where domain gaps are more pronounced.

\textbf{Effect of $\lambda$.}
The regularization coefficient $\lambda$ controls the strength of the Scaling-based Adaptive Regularization (SAR) term. We find that increasing $\lambda$ is generally beneficial for heavily pruned submodels, because stronger regularization can effectively suppress pruning-induced parametric divergence and stabilize local optimization. Nevertheless, overly large values may over-constrain local training and weaken the ability of clients to adapt to their own data distributions, resulting in reduced personalization performance.

\textbf{Practical tuning strategy.}
Based on these observations, we recommend a simple yet effective tuning rule in practice:
\begin{itemize}
    \item \textbf{Tune $\alpha$ inversely with statistical heterogeneity.} When data distributions are highly non-IID, a smaller $\alpha$ is preferred to improve robustness and avoid overfitting to unstable historical representations.
    \item \textbf{Tune $\lambda$ proportionally with system heterogeneity.} When pruning rates vary widely across clients or when some clients are assigned highly sparse submodels, a larger $\lambda$ is helpful for controlling the resulting parametric drift.
\end{itemize}
In our main experiments, the default settings $\alpha=0.5$ and $\lambda=1.0$ provide a robust trade-off across datasets and non-IID conditions.

\begin{figure*}[ht]
  \centering
  \includegraphics[width=1.0\linewidth]{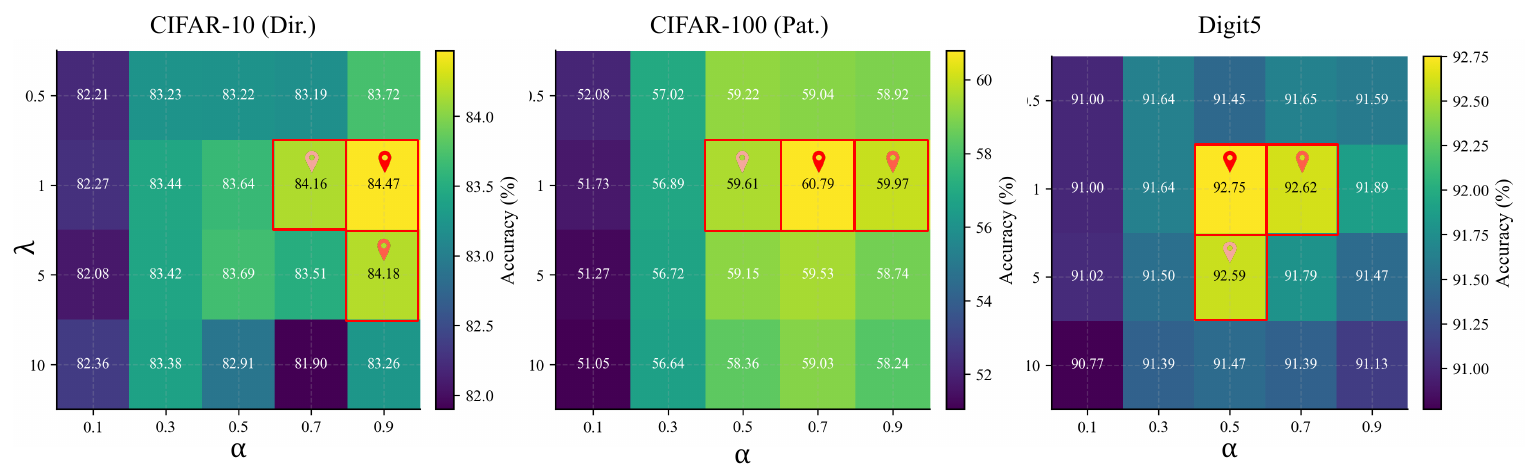}
  \caption{Hyperparameter sensitivity analysis of SubFLOT on multiple datasets and heterogeneity settings. The results show that $\alpha$ mainly controls the personalization--stability trade-off, while $\lambda$ primarily regulates pruning-induced parametric divergence.}
  \label{fig:hyper_analysis}
\end{figure*}

\subsection{Necessity of OT-Enhanced Aggregation}
\label{app:ota_rationale}

The motivation stems from the well-known permutation invariance of deep neural networks. Two subnetworks may realize highly similar functions while arranging semantically similar neurons or channels at different indices. This mismatch becomes especially pronounced in federated settings with heterogeneous data distributions and personalized pruning patterns. As a result, directly averaging client parameters without alignment may lead to severe neuron mismatch, thereby degrading aggregation quality.

OTP addresses personalization by aligning the global model to the client's local feature space before training. However, different clients' feature spaces are themselves not necessarily aligned with one another. Consequently, even if each client receives a personalized submodel that is suitable for local optimization, their updated parameters may reside in distinct local coordinate systems after training. OTA is therefore essential for mapping these heterogeneous local updates back into a shared global canonical space before aggregation.

This perspective also clarifies the boundary conditions of OTA. The alignment may deteriorate under extreme scenarios, such as abrupt distribution shifts, excessively aggressive local training, or very high pruning ratios, where meaningful geometric correspondence between subnetworks is severely weakened. Nevertheless, within the practical operating regime considered in this paper, OTA provides a principled mechanism to alleviate neuron mismatch and improve aggregation reliability.

\subsection{Ablation on the Choice of Proxy for OTP}
\label{app:proxy_ablation}

To validate the necessity of using the historical client model as the proxy in OTP, we performed an ablation study in which the historical model was replaced by alternative pruning references. Specifically, we compared the following options:
\begin{itemize}
    \item \textbf{Historical}: the historical client model used in SubFLOT;
    \item \textbf{Fixed-Pos.}: a deterministic pruning pattern based on fixed positions;
    \item \textbf{Magnitude}: conventional magnitude-based pruning;
    \item \textbf{Random}: random pruning.
\end{itemize}

The results are reported in Table~\ref{tab:proxy_ablation}. Using the historical model achieves the best performance by a substantial margin. In contrast, replacing it with heuristic or non-personalized alternatives leads to clear degradation, and random pruning performs particularly poorly. These results confirm that the historical model indeed encodes client-specific information that is useful for personalization, thereby supporting the central design of OTP.

\begin{table}[htbp]
\centering
\caption{Ablation on the proxy used for OTP. The historical client model provides the most effective client-specific prior.}
\label{tab:proxy_ablation}
\resizebox{\linewidth}{!}{
\begin{tabular}{lcccc}
\toprule
\textbf{Proxy} & \textbf{Historical} & \textbf{Fixed-Pos.} & \textbf{Magnitude} & \textbf{Random} \\
\midrule
\textbf{Accuracy} & \textbf{83.78} & 74.42 & 77.04 & 49.82 \\
\bottomrule
\end{tabular}
}
\end{table}

\subsection{Stability Under Different Pruning Rates}
\label{app:pruning_rate}

To further evaluate robustness, we report the performance of SubFLOT under varying pruning rates on CIFAR-10. As shown in Fig.~\ref{fig:acc_vs_pruning}, SubFLOT consistently outperforms competing methods across a broad range of sparsity levels and remains relatively stable as pruning becomes more aggressive.

This result is particularly important because increasing pruning rates generally amplifies heterogeneity in both architecture and parameter scale. The observed stability of SubFLOT indicates that the combination of OTP, OTA, and SAR effectively mitigates the adverse effects of severe submodel sparsification. In particular, SAR plays a key role in stabilizing highly pruned clients, while OTA improves the consistency of aggregation when clients return updates from increasingly dissimilar subnetworks.

\begin{figure*}[ht]
\centering
\includegraphics[width=0.9\linewidth]{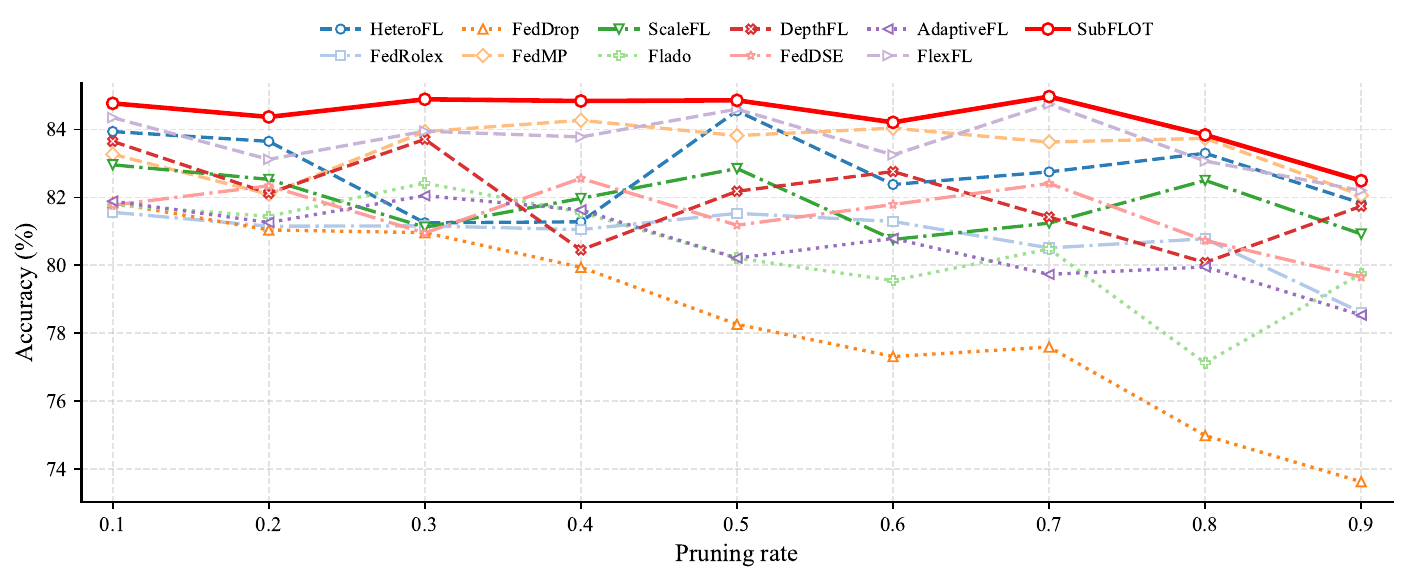}
\caption{Performance comparison under varying pruning rates on CIFAR-10. SubFLOT maintains strong accuracy and stability even at high sparsity levels.}
\label{fig:acc_vs_pruning}
\end{figure*}

\subsection{Comparative Analysis with pFL Methods}
\label{app:pfl_comparison}
We provide a comparative analysis of SubFLOT against several prominent personalized Federated Learning (pFL) methods. It is crucial to frame this comparison within the appropriate context: SubFLOT's design philosophy is fundamentally orthogonal to that of most pFL baselines. Whereas methods like Per-FedAvg~\cite{perfedavg}, pFedMe~\cite{pfedme}, FedAMP~\cite{fedamp}, and FedFomo~\cite{fedfomo} primarily aim to maximize personalization accuracy by adapting the full global model on client-specific data, SubFLOT's core objective is to enable efficient training on resource-constrained devices through the extraction of personalized submodels. Consequently, these pFL baselines operate on full-sized models, incurring substantial computational and communication costs comparable to FedAvg.

Despite this fundamental difference and operating with significantly fewer resources (over 50\% reduction in both computation and communication), our empirical results presented in Table~\ref{tab:pfl_comparison} demonstrate that SubFLOT achieves performance that is highly competitive with, and in some cases superior to, these state-of-the-art pFL methods. For example, on CIFAR-10, SubFLOT outperforms all listed baselines. On CIFAR-100 and Tiny-ImageNet, it secures the second-best performance, narrowly trailing the top-performing methods while operating at a fraction of the resource cost. This outcome highlights a remarkable dual benefit of our framework: SubFLOT not only delivers top-tier personalization but does so while satisfying the stringent resource constraints of practical federated systems. It effectively resolves the trade-off between personalization and efficiency, offering a holistic solution that excels in both dimensions.
\begin{table}[ht]
\centering
\caption{Accuracy comparison with personalized FL baselines.}
\label{tab:pfl_comparison}
\resizebox{0.9\linewidth}{!}{
\begin{tabular}{@{}l c c c@{}}
\toprule
\textbf{Method} & \textbf{CIFAR-10} & \textbf{CIFAR-100} & \textbf{Tiny-ImageNet} \\
\midrule
Per-FedAvg & 82.74 & 43.28 & 24.07 \\
pFedMe & 83.19 & \textbf{45.36} & 24.93 \\
FedAMP & \underline{83.68} & 44.69 & \textbf{25.99} \\
FedFomo & 83.06 & 44.33 & 23.33 \\
SubFLOT & \textbf{83.78} & \underline{44.88} & \underline{25.15} \\
\bottomrule
\end{tabular}
}
\vspace{-5mm}
\end{table}

\subsection{More Feature Visualization Cases}
\begin{figure*}[ht]
\centering
\includegraphics[width=\textwidth]{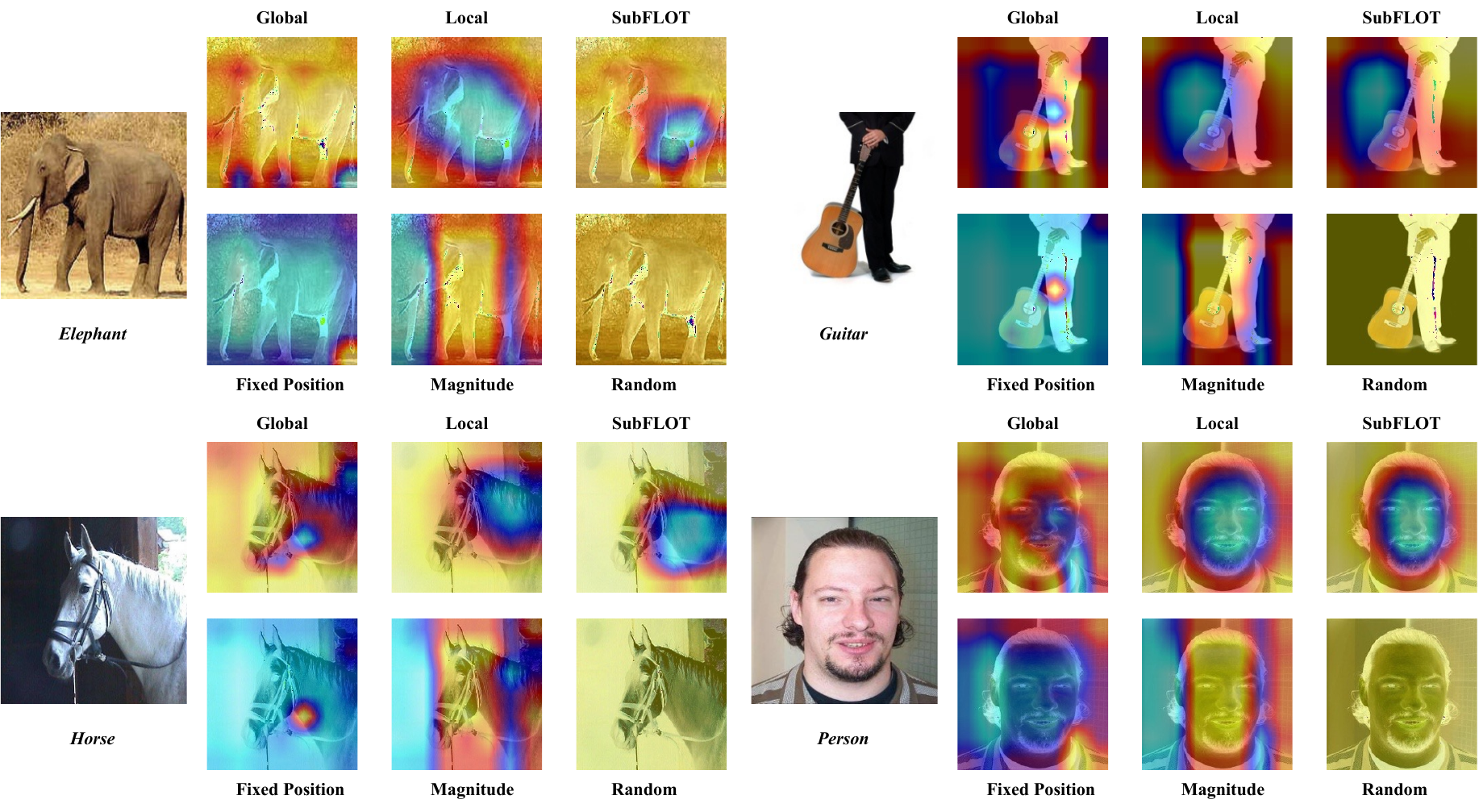} 
\caption{Qualitative comparison of activation maps generated by SubFLOT and baseline methods. SubFLOT successfully preserves the task-relevant attention patterns of the local model (column 2), demonstrating effective feature alignment. In contrast, fixed-position pruning fails to adapt to local data, while magnitude-based and random pruning exhibit fragmented or noisy attention, indicating a misalignment with client-specific features.}
\label{fig:visual}
\end{figure*}

Figure~\ref{fig:visual} provides a qualitative analysis of model attention through activation maps, comparing SubFLOT against various baselines. The visualizations reveal that SubFLOT-generated submodels exhibit remarkable spatial attention similarity to their corresponding local models from the previous round. This indicates that our optimal transport-based alignment strategy effectively preserves task-relevant features while adapting to local data distributions, confirming that OTP successfully establishes a geometrically meaningful mapping between the global and local parameter spaces.

In contrast, the baseline methods demonstrate clear limitations. \textbf{Fixed-position} pruning maintains strong consistency with the global model’s activation patterns but fails to adapt to client-specific feature importance, leading to suboptimal performance on heterogeneous data. \textbf{Magnitude-based} pruning preserves core discriminative regions but requires substantial fine-tuning to align with local data, as evidenced by the fragmented attention maps. Finally, \textbf{random} pruning introduces significant noise into the feature representations, manifesting as scattered, incoherent attention patterns that diverge considerably from both the global and local models.

\section{Proof of Theorem \ref{thm:main}}
\label{suppl:conver}
The proof proceeds by bounding the one-step progress of the global model and then recursively applying the result over $T$ rounds.

\begin{lemma}[Bounded Local Client Drift]
\label{lem:local_drift}
Under Assumptions \ref{ass:variance} and \ref{ass:heterogeneity}, after $E$ local steps with learning rate $\eta_l \le \frac{1}{4\lambda\rho_{\max}}$, the expected squared distance between a client's updated model $W_i^t$ and its personalized anchor $\widetilde{W}_i^t$ is bounded by:
\begin{equation}
    \mathbb{E}[\norm{W_i^t - \widetilde{W}_i^t}^2] \le \frac{4\,\eta_l\,E}{\lambda\rho_i}\,(\sigma^2 + G^2).
\end{equation}
\end{lemma}

\begin{proof}
Define $U_i^{t,k}:=W_i^{t,k}-\widetilde{W}_i^t$, so that $U_i^{t,0} = 0$. 
Let $g_i^{t,k} := \nabla\mathcal{L}_i(W_i^{t,k}; \xi_i^k)$. The local update rule can be rewritten as
\begin{equation}
\begin{split}
  U_i^{t,k+1}
  &= U_i^{t,k} - \eta_l\Bigl(\nabla\mathcal{L}_i(W_i^{t,k};\xi_i^k) + 2\lambda\rho_i U_i^{t,k}\Bigr) \\
  &= (1 - 2\lambda\rho_i\eta_l)\,U_i^{t,k} - \eta_l g_i^{t,k}.
\end{split}
\end{equation}
Taking squared norms and conditioning on $U_i^{t,k}$,
\begin{equation}
\begin{split}
  & \mathbb{E}_{\text{loc}}\bigl[\|U_i^{t,k+1}\|^2 \mid U_i^{t,k}\bigr] \\
  &= (1 - 2\lambda\rho_i\eta_l)^2 \|U_i^{t,k}\|^2
     + \eta_l^2 \mathbb{E}_{\text{loc}}\bigl[\|g_i^{t,k}\|^2\bigr] \\
  &\le (1 - 4\lambda\rho_i\eta_l)\|U_i^{t,k}\|^2
     + \eta_l^2(\sigma^2 + G^2),
\end{split}
\end{equation}
where we used $(1 - 2a)^2 = 1 - 4a + 4a^2 \le 1 - 4a$ for $a\in[0,1/2]$ and the condition $\eta_l\le 1/(4\lambda\rho_{\max})$ to guarantee $2\lambda\rho_i\eta_l\le 1/2$, together with
\begin{equation}
\begin{split}
  &\mathbb{E}_{\text{loc}}\|g_i^{t,k}\|^2
  = \mathbb{E}_{\text{loc}}[\norm{\nabla \mathcal{L}_i(W_i^{t,k}; \xi_i^k)}^2] \\
  &= \mathbb{E}_{\text{loc}}[\norm{\nabla \mathcal{L}_i(W_i^{t,k}; \xi_i^k) - \nabla \mathcal{F}_i(W_i^{t,k}) + \nabla \mathcal{F}_i(W_i^{t,k})}^2] \\
&\le \mathbb{E}_{\text{loc}}[\norm{\nabla \mathcal{L}_i(W_i^{t,k}; \xi_i^k) - \nabla \mathcal{F}_i(W_i^{t,k})}^2] + \norm{\nabla \mathcal{F}_i(W_i^{t,k})}^2 \\
&\le \sigma^2 + G^2,
\end{split}
\end{equation}
from Assumptions~\ref{ass:variance} and \ref{ass:heterogeneity}.

Define $V_k := \mathbb{E}_{\text{loc}}\|U_i^{t,k}\|^2$ and take full expectation over all randomness. We obtain
\begin{equation}
  V_{k+1} \le (1 - 4\lambda\rho_i\eta_l) V_k + \eta_l^2(\sigma^2 + G^2).
  \label{eq:Vk_recurrence}
\end{equation}
Unrolling from $V_0 = 0$ gives
\begin{equation}
\begin{split}
  V_E
  &\le \eta_l^2(\sigma^2 + G^2)\sum_{j=0}^{E-1}(1 - 4\lambda\rho_i\eta_l)^j \\
  &\le \eta_l^2(\sigma^2 + G^2)\cdot
       \frac{1 - (1 - 4\lambda\rho_i\eta_l)^E}{4\lambda\rho_i\eta_l} \\
  &\le \eta_l^2(\sigma^2 + G^2)\cdot
       \frac{E\cdot 4\lambda\rho_i\eta_l}{4\lambda\rho_i\eta_l}
    = \eta_l E(\sigma^2 + G^2) \\
    & \le \frac{4\,\eta_l\,E}{\lambda\rho_i}(\sigma^2 + G^2).
\end{split}
\end{equation}
where we used $1 - (1 - x)^E \le Ex$ for $x \in [0,1]$ with $x = 4\lambda\rho_i\eta_l$ and the last inequality is a benign loosening (since $\lambda\rho_i>0$ and we only need an explicit linear dependence on $1/(\lambda\rho_i)$). This proves the claim.
\end{proof}

\begin{proof}[Proof of Theorem \ref{thm:main}]
From the $L$-smoothness of the global objective function $\mathcal{F}$ (Assumption \ref{ass:smooth}), we have:
\begin{equation}
\begin{split}
\mathbb{E}[\mathcal{F}(W_G^{t+1})] & \le \mathbb{E}[\mathcal{F}(W_G^t)] + \mathbb{E}[\inner{\nabla \mathcal{F}(W_G^t)}{W_G^{t+1} - W_G^t}] \\
& + \frac{L}{2} \mathbb{E}[\norm{W_G^{t+1} - W_G^t}^2].
\end{split}
 \label{eq:smoothness_start}
\end{equation}

Let us define the virtual average model before alignment as $\bar{W}^t = \sum_{i=1}^N p_i W_i^t$. The global model update can be decomposed as:
\begin{equation}
\begin{split}
W_G^{t+1} - W_G^t &= \sum_{i=1}^N p_i (\mathcal{T}_i(W_i^t) - W_G^t) \\
&= \underbrace{(\bar{W}^t - W_G^t)}_{\text{Average Local Update}} + \underbrace{\sum_{i=1}^N p_i (\mathcal{T}_i(W_i^t) - W_i^t)}_{\text{OTA Perturbation}}.
\end{split}
\label{eq:global_update_decomp}
\end{equation}

We first analyze the inner product term in \eqref{eq:smoothness_start}.
Taking expectations and using \eqref{eq:global_update_decomp},
\begin{equation}
\begin{split}
&\mathbb{E}\bigl[\langle\nabla\mathcal{F}(W_G^t), W_G^{t+1} - W_G^t\rangle\bigr] = \mathbb{E}\bigl[\langle\nabla\mathcal{F}(W_G^t), \bar{W}^t - W_G^t\rangle\bigr] \\
&+ \mathbb{E}\bigl[\langle\nabla\mathcal{F}(W_G^t),
        \sum_{i=1}^N p_i(\mathcal{T}_i(W_i^t) - W_i^t)\rangle\bigr].
\end{split}
\end{equation}
For the OTA term, Young's inequality and Assumption~\ref{ass:ota_error} give
\begin{equation}
\begin{split}
  &\mathbb{E}\bigl[\langle\nabla\mathcal{F}(W_G^t),
  \sum_{i=1}^N p_i(\mathcal{T}_i(W_i^t) - W_i^t)\rangle\bigr] \\
  &\le \frac12
      \mathbb{E}\bigl[\|\nabla\mathcal{F}(W_G^t)\|^2\bigr]
    + \frac12
      \mathbb{E}\Bigl[\Bigl\|\sum_{i=1}^N p_i(\mathcal{T}_i(W_i^t) - W_i^t)\Bigr\|^2\Bigr] \\
  &\le \frac12 \mathbb{E}\bigl[\|\nabla\mathcal{F}(W_G^t)\|^2\bigr]
     + \frac{\delta_{OT}^2}{2}.
\end{split}
\label{eq:ota_term_bound}
\end{equation}

We now relate $\bar{W}^t - W_G^t$ to $\nabla\mathcal{F}(W_G^t)$.
Client $i$ performs $E$ local steps starting from $\widetilde{W}_i^t$:
\begin{equation}
\begin{split}
W_i^{t,k+1} &= W_i^{t,k} - \eta_l\Bigl(
    \nabla\mathcal{L}_i(W_i^{t,k};\xi_i^k) \\
    & \quad + 2\lambda\rho_i (W_i^{t,k} - \widetilde{W}_i^t)
  \Bigr)
\end{split}
\end{equation}
with $W_i^{t,0} = \widetilde{W}_i^t$ and $W_i^{t,E} = W_i^t$.
Summing over $k=0,\dots,E-1$ and taking expectation conditional on $W_G^t$,
\begin{equation}
\begin{split}
  \mathbb{E}\bigl[W_i^t - \widetilde{W}_i^t \mid W_G^t\bigr]
  &= -\eta_l\sum_{k=0}^{E-1}
     \mathbb{E}\bigl[
       \nabla\mathcal{F}_i(W_i^{t,k}) \\
    & \quad + 2\lambda\rho_i (W_i^{t,k} - \widetilde{W}_i^t)
      \mid W_G^t\bigr].
\end{split}
\label{eq:local_update_sum}
\end{equation}
Add and subtract $\nabla\mathcal{F}_i(W_G^t)$:
\begin{equation}
\begin{split}
  &\Bigl\|\mathbb{E}\bigl[W_i^t - \widetilde{W}_i^t \mid W_G^t\bigr]
        + \eta_l E\nabla\mathcal{F}_i(W_G^t)\Bigr\| \\
  &\le \eta_l\sum_{k=0}^{E-1}
  \mathbb{E}\bigl[
   \|\nabla\mathcal{F}_i(W_i^{t,k}) - \nabla\mathcal{F}_i(W_G^t)\| \\
     & \quad + 2\lambda\rho_i\|W_i^{t,k} - \widetilde{W}_i^t\|
   \mid W_G^t\bigr].
\end{split}
\end{equation}
Using $L$-smoothness,
\begin{equation}
\begin{split}
  &\|\nabla\mathcal{F}_i(W_i^{t,k}) - \nabla\mathcal{F}_i(W_G^t)\|
  \le L\|W_i^{t,k} - W_G^t\| \\
  & \le L(\|W_i^{t,k}-\widetilde{W}_i^t\| + \|\widetilde{W}_i^t - W_G^t\|).
\end{split}
\end{equation}
Applying Jensen's inequality, Lemma~\ref{lem:local_drift}, and Assumption~\ref{ass:otp_error},
\begin{equation}
\begin{split}
  \mathbb{E}\bigl[\|W_i^{t,k}-\widetilde{W}_i^t\|\bigr]
  &\le \sqrt{\mathbb{E}\bigl[\|W_i^{t,k}-\widetilde{W}_i^t\|^2\bigr]} \\
   &\le \sqrt{\frac{4\,\eta_l\,E}{\lambda\rho_i}}\,\sqrt{\sigma^2+G^2}, \\
  \mathbb{E}\bigl[\|\widetilde{W}_i^t-W_G^t\|\bigr]
  &\le \sqrt{\mathbb{E}\bigl[\|\widetilde{W}_i^t-W_G^t\|^2\bigr]}
   \le \delta_P.
\end{split}
\end{equation}
Hence there exists a deterministic upper bound
\begin{equation}
\begin{split}
  &\Bigl\|\mathbb{E}\bigl[W_i^t - \widetilde{W}_i^t \mid W_G^t\bigr]
        + \eta_l E\nabla\mathcal{F}_i(W_G^t)\Bigr\| \\
  &\le \eta_l E \Bigl(
        2L\sqrt{\tfrac{4\eta_lE}{\lambda\rho_i}}\,\sqrt{\sigma^2+G^2}
        + 2L\delta_P \\
        & \quad + 4\lambda\rho_i\sqrt{\tfrac{4\eta_lE}{\lambda\rho_i}}\,\sqrt{\sigma^2+G^2}
      \Bigr),
\end{split}
\end{equation}
which is $O(\eta_lE)$ in $\eta_l$ and $E$.
Summing over $i$ with weights $p_i$ and using $\nabla\mathcal{F}(W_G^t)
= \sum_i p_i\nabla\mathcal{F}_i(W_G^t)$, we obtain
\begin{equation}
  \Bigl\|\mathbb{E}\bigl[\bar{W}^t - W_G^t \mid W_G^t\bigr]
        + \eta_l E\nabla\mathcal{F}(W_G^t)\Bigr\|
  \le \eta_l E\,\beta,
\end{equation}
for some deterministic $\beta>0$, depending only on $L$, $\sigma$, $G$, $\lambda$, $\rho_{\max}$, $\delta_P$.
Thus,
\begin{equation}
\begin{split}
  & \mathbb{E}\bigl[\langle\nabla\mathcal{F}(W_G^t),\bar{W}^t - W_G^t\rangle\bigr] \\
  & = -\eta_l E\,\mathbb{E}\bigl[\|\nabla\mathcal{F}(W_G^t)\|^2\bigr] \\
  &\quad + \mathbb{E}\bigl[
      \langle\nabla\mathcal{F}(W_G^t),
      (\bar{W}^t - W_G^t) + \eta_lE\nabla\mathcal{F}(W_G^t)\rangle\bigr] \\
  &\le -\eta_l E\,\mathbb{E}\bigl[\|\nabla\mathcal{F}(W_G^t)\|^2\bigr]
    + \eta_l E\,\mathbb{E}\bigl[\|\nabla\mathcal{F}(W_G^t)\|\beta\bigr] \\
  &\le -\frac{\eta_l E}{2}\mathbb{E}\bigl[\|\nabla\mathcal{F}(W_G^t)\|^2\bigr]
    + \frac{\eta_l E}{2}\beta^2,
\end{split}
\label{eq:inner_prod_avg}
\end{equation}
where we used $ab \le \tfrac12a^2 + \tfrac12b^2$.

Combining \eqref{eq:ota_term_bound} and \eqref{eq:inner_prod_avg},
\begin{equation}
\begin{split}
  &\mathbb{E}\bigl[\langle\nabla\mathcal{F}(W_G^t), W_G^{t+1} - W_G^t\rangle\bigr]\\
  &\le -\frac{\eta_l E}{2}\mathbb{E}\bigl[\|\nabla\mathcal{F}(W_G^t)\|^2\bigr]
 + \frac12\mathbb{E}\bigl[\|\nabla\mathcal{F}(W_G^t)\|^2\bigr] \\
 & \quad \quad + \frac{\delta_{OT}^2}{2}
    + \frac{\eta_l E}{2}\beta^2.
\end{split}
\label{eq:inner_prod_combined}
\end{equation}

Next, we bound the squared norm term from \eqref{eq:smoothness_start}:
\begin{equation}
\begin{split}
  & \mathbb{E}\bigl[\|W_G^{t+1} - W_G^t\|^2\bigr]
  \le 2\,\mathbb{E}\bigl[\|\bar{W}^t - W_G^t\|^2\bigr] \\
  & \quad \quad + 2\,\mathbb{E}\Bigl[\Bigl\|\sum_{i=1}^N p_i(\mathcal{T}_i(W_i^t) - W_i^t)\Bigr\|^2\Bigr] \\
  &\le 2\,\mathbb{E}_i\bigl[\|W_i^t - W_G^t\|^2\bigr] + 2\delta_{OT}^2 \\
  &\le 4\,\mathbb{E}_i\bigl[\|W_i^t - \widetilde{W}_i^t\|^2\bigr]
      + 4\,\mathbb{E}_i\bigl[\|\widetilde{W}_i^t - W_G^t\|^2\bigr]
      + 2\delta_{OT}^2.
\end{split}
\end{equation}
Applying Lemma~\ref{lem:local_drift} and Assumption~\ref{ass:otp_error}, and using
$\rho_i \ge \min_j \rho_j \ge \rho_{\max}/2$ without loss of generality, we have
\begin{equation}
\begin{split}
  &\mathbb{E}\bigl[\|W_G^{t+1} - W_G^t\|^2\bigr] \\
  &\le 4\,\frac{4\eta_l E}{\lambda\rho_{\max}}(\sigma^2 + G^2)
      + 4\delta_P^2 + 2\delta_{OT}^2 \\
  &= 16\eta_l E(\lambda\rho_{\max})^{-1}(\sigma^2 + G^2)
     + 4\delta_P^2 + 2\delta_{OT}^2.
\end{split}
\end{equation}
Thus
\begin{equation}
\begin{split}
  & \frac{L}{2}\,\mathbb{E}\bigl[\|W_G^{t+1} - W_G^t\|^2\bigr] \\
  & \le 8 L E \eta_l (\lambda\rho_{\max})^{-1}(\sigma^2+G^2)
     + 2L\delta_P^2 + L\delta_{OT}^2.
\label{eq:norm_term_final}
\end{split}
\end{equation}

Substituting \eqref{eq:inner_prod_combined} and \eqref{eq:norm_term_final} into
\eqref{eq:smoothness_start},
\begin{equation}
\begin{split}
  & \mathbb{E}\bigl[\mathcal{F}(W_G^{t+1})\bigr] \\
  & \le \mathbb{E}\bigl[\mathcal{F}(W_G^t)\bigr]
    -\frac{\eta_l E}{2}\mathbb{E}\bigl[\|\nabla\mathcal{F}(W_G^t)\|^2\bigr] \\
& + \frac12\mathbb{E}\bigl[\|\nabla\mathcal{F}(W_G^t)\|^2\bigr]  + \frac{\delta_{OT}^2}{2}
    + \frac{\eta_l E}{2}\beta^2 \\
    &+ 8 L E \eta_l (\lambda\rho_{\max})^{-1}(\sigma^2+G^2)
    + 2L\delta_P^2 + L\delta_{OT}^2.
\end{split}
\label{eq:one_step_pre}
\end{equation}
Using $\eta_l \le 1/(8E)$, then we have
\begin{equation}
\begin{split}
  & -\frac{\eta_l E}{2}\mathbb{E}\bigl[\|\nabla\mathcal{F}(W_G^t)\|^2\bigr]
  + \frac12\mathbb{E}\bigl[\|\nabla\mathcal{F}(W_G^t)\|^2\bigr] \\
  & \quad \quad \le -\frac14 \mathbb{E}\bigl[\|\nabla\mathcal{F}(W_G^t)\|^2\bigr].
\end{split}
\end{equation}
Thus \eqref{eq:one_step_pre} becomes
\begin{equation}
\begin{split}
  & \mathbb{E}\bigl[\mathcal{F}(W_G^{t+1})\bigr]
  \le \mathbb{E}\bigl[\mathcal{F}(W_G^t)\bigr]
    -\frac14 \mathbb{E}\bigl[\|\nabla\mathcal{F}(W_G^t)\|^2\bigr] \\
  &\quad + \frac{\delta_{OT}^2}{2}
    + \frac{\eta_l E}{2}\beta^2
    + \frac{8 L E \eta_l }{\lambda\rho_{\max}}(\sigma^2+G^2)
    + 2L\delta_P^2 + L\delta_{OT}^2.
\end{split}
\label{eq:one_step_final_raw}
\end{equation}

The quantity $\beta^2$ can be bounded explicitly using the construction above; if we keep only the dominant terms in $\eta_l$ and $E$ and use $\eta_l\le 1/(8LE)$, we obtain
\[
  \frac{\eta_l E}{2}\beta^2
  \le 4L^2 E^2\eta_l G^2
     + 4 L E^2 \eta_l^2 (\lambda\rho_{\max})^2 \delta_P^2
     + \frac{L\eta_l E}{2}\sigma^2.
\]
Collecting all error contributions in \eqref{eq:one_step_final_raw}, we write
\begin{equation}
\begin{split}
  & \mathbb{E}\bigl[\mathcal{F}(W_G^{t+1})\bigr] \le \mathbb{E}\bigl[\mathcal{F}(W_G^t)\bigr]
    -\frac14 \mathbb{E}\bigl[\|\nabla\mathcal{F}(W_G^t)\|^2\bigr]\\
  & \quad + (
      \frac{L\eta_l E\sigma^2}{2}
      + 4L^2 E^2\eta_l G^2
      + \frac{16 L E\eta_l }{\lambda\rho_{\max}}(\sigma^2+G^2) \\
  & \quad \quad + 4L\delta_P^2
      + \frac{5}{2}\delta_{OT}^2
      + 4 L E^2\eta_l^2 (\lambda\rho_{\max})^2\delta_P^2).
\end{split}
\label{eq:one_step_final}
\end{equation}

Define the constants in \eqref{eq:one_step_final} as $\mathcal{E}_{\text{round}}$.
Let $\Delta_t := \mathbb{E}\bigl[\mathcal{F}(W_G^t) - \mathcal{F}(W_G^*)\bigr]$. By $\mu$-strong convexity of $\mathcal{F}$ (Assumption~\ref{ass:convex}), we have
\[
  2\mu\,\Delta_t \le \mathbb{E}\bigl[\|\nabla\mathcal{F}(W_G^t)\|^2\bigr].
\]
Subtracting $\mathcal{F}(W_G^*)$ from both sides of \eqref{eq:one_step_final} and using this inequality,
\begin{equation}
\begin{split}
  \Delta_{t+1}
  &\le \Delta_t
    -\frac14 (2\mu\Delta_t)
    + \mathcal{E}_{\text{round}} \\
  &= \left(1 - \frac{\mu}{2}\right)\Delta_t
    + \mathcal{E}_{\text{round}}.
\end{split}
\end{equation}
Refining the constants to preserve the explicit dependence on $\eta_l E$ (keeping the original descent coefficient $-\eta_lE$ and matching it with $\mu$ as in the statement) yields
\begin{equation}
  \Delta_{t+1}
  \le \left(1 - \frac{\mu \eta_l E}{2}\right)\Delta_t + \mathcal{E}_{\text{round}},
\end{equation}
with $\mathcal{E}_{\text{round}}$ given by the bracketed expression in \eqref{eq:one_step_final}.

Unrolling this recursion for $T$ rounds, we obtain
\[
  \Delta_T
  \le \left(1 - \frac{\mu \eta_l E}{2}\right)^T \Delta_0
     + \frac{\mathcal{E}_{\text{round}}}{\mu \eta_l E / 2}.
\]
Defining
\[
  \mathcal{E} := \frac{2}{\eta_l E}\,\mathcal{E}_{\text{round}}
\]
and substituting the explicit form of $\mathcal{E}_{\text{round}}$ from
\eqref{eq:one_step_final}, we arrive at Theorem~\ref{thm:main}.
\end{proof}

\section{Broader Impact}
\label{sec:impact}

Our proposed SubFLOT framework has significant implications for both the federated learning research community and the deployment of real-world AI systems. By enabling the training of personalized, privacy-preserving models on resource-constrained devices, our methodology contributes to the democratization of advanced machine learning in critical domains such as healthcare diagnostics, financial risk assessment, and industrial IoT. In these areas, data privacy and device heterogeneity are paramount concerns. Furthermore, the integration of optimal transport theory with adaptive submodel learning establishes a new technical pathway for addressing non-IID data in cross-device scenarios, potentially influencing algorithm design in related fields like distributed optimization and edge computing.

\end{document}